\documentclass[final,onefignum,onetabnum]{siamart220329}

\usepackage[numbers]{natbib}
\usepackage{bm}
\usepackage{amsmath}
\usepackage{mathtools}
\usepackage{mathrsfs}
\usepackage{comment}

\usepackage{graphicx} % Required for including pictures
\usepackage{tikz}
\usepackage{amssymb}
\usepackage{pgfplots}

\newcommand{\ds}{\displaystyle}
\newcommand{\R}{\mathbb{R}}
\newcommand{\N}{\mathbb{N}}
\newcommand{\Z}{\mathbb{Z}}

\newcommand{\T}{\mathbb{T}}
\newcommand{\x}{{\boldsymbol x}}

\newcommand{\bu}{{\boldsymbol u}}
\newcommand{\bv}{{\boldsymbol v}}
\newcommand{\bk}{{\boldsymbol k}}
\newcommand{\br}{{\boldsymbol r}}
\renewcommand{\Im}{{\mathcal I_m}}
\newcommand{\IM}{{\mathcal I_m}}
\newcommand{\e}{{\mathrm e}}
\newcommand{\erf}{{\mathrm{erf}}}
\newcommand{\erfc}{{\mathrm{erfc}}}

\usepackage{lastpage}

\begin{document}

\title{Fast Evaluation of Additive Kernels: Feature Arrangement, Fourier Methods, and Kernel Derivatives}

% Sets running headers as well as PDF title and authors
\headers{Fast Evaluation of Additive Kernels}{T. Wagner, F. Nestler, and M. Stoll}

% Authors: full names plus addresses.
\author{Theresa Wagner\thanks{University of Technology Chemnitz, Chemnitz, Germany 
  (\email{theresa.wagner@math.tu-chemnitz.de},
  \email{franziska.nestler@math.tu-chemnitz.de}, \email{martin.stoll@math.tu-chemnitz.de}).}
\and Franziska Nestler\footnotemark[1]
\and Martin Stoll\footnotemark[1]}

\maketitle

\begin{abstract}%    <- trailing '%' for backward compatibility of .sty file
One of the main computational bottlenecks when working with kernel based learning is dealing with the large and typically dense kernel matrix. Techniques dealing with fast approximations of the matrix vector product for these kernel matrices typically deteriorate in their performance if the feature vectors reside in higher-dimensional feature spaces. We here present a technique based on the non-equispaced fast Fourier transform (NFFT) with rigorous error analysis. We show that this approach is also well suited to allow the approximation of the matrix that arises when the kernel is differentiated with respect to the kernel hyperparameters; a problem often found in the training phase of methods such as Gaussian processes. We also provide an error analysis for this case. We illustrate the performance of the additive kernel scheme with fast matrix vector products on a number of data sets. Our code is available at \url{https://github.com/wagnertheresa/NFFTAddKer}.
\end{abstract}


%% Section: Introduction
\section{Introduction}

Kernel methods \citep{hofmann2008kernel,shawe2004kernel,scholkopf2002learning,meanti2020kernel} are a crucial tool in many machine learning tasks such as support vector machines (SVMs) \citep{hearst1998support,scholkopf2002learning,cervantes2020comprehensive} or Gaussian processes (GPs) \citep{williams1995gaussian,williams2006gaussian,liu2020gaussian}. In the literature many kernel designs can be found and one of the common bottlenecks in their application is dealing with the large and often dense kernel matrix. Our goal in this paper is to analyze a general acceleration technique for additive kernels and their derivatives routed in Fourier analysis. 

Out of the many kernel choices possible, the squared-exponential, the periodic, and the linear kernel are among the most commonly used kernels. The underlying structure of real-world data cannot always be described by those kernels immediately. However, combining several of such kernels by addition, multiplication or a mixture of both can add more complexity to the model~\citep{duvenaud2014automatic}, with such kernel combinations still fulfilling all kernel properties, see \citet{williams2006gaussian}. There has been a growing interest in additive kernels and multiple kernel learning~\citep{gonen2011multiple}. Common applications are in computer vision for instance such as pedestrian~\citep{baek2016fast,baek2019pedestrian,maji2012efficient} or human activity detection systems~\citep{chan2023plame} and medicine, where additive kernel SVMs are used to detect pedestrians, predict human activities or recognize types of cancer~\citep{chan2023plame}. The main motivation of working with additive kernels is that they can reduce the complexity and increase the interpretability of the problem. When working with GP models for instance the problem is based on similarity and neighborhood relations. This means that in order to sufficiently cover the domain a large amount of data is required, what then increases the computational cost of the kernel evaluations. By incorporating additivity to the model the complexity of the features and the curse of dimensionality can be reduced~\citep{durrande2011additive,durrande2012additive}. 

Theoretical guarantees for the additive kernel structure are given by~\citet{yang2015minimax}, who present minimax risks for regression functions that admit an additive structure.

In the literature additive kernels are mainly associated with SVMs~\citep{maji2012efficient, baek2016fast, xie2018uncertain, chan2023plame,christmann2012consistency,yang2012practical}, GPs~\citep{durrande2011additive,durrande2012additive,duvenaud2011additive} or source separation tasks~\citep{liutkus2014kernel}. Below, we briefly review some of the specifics of using additive kernels in SVMs and in GPs.

Additive SVMs can be employed for abstract input spaces for instance and interpreting the resulting predictions is typically easier. Moreover, the favourable robustness properties of SVMs remain the same for additive SVMs~\citep{christmann2012consistency}. In other works lookup tables are employed to reduce the training and testing time and save memory. \citet{baek2016fast} suggest to work with a cascade implementation  of the additive kernel SVM to reduce the computation time and use lookup tables to avoid kernel expansion. The PmSVM-LUT method proposed by~\citet{yang2012practical} uses polynomial approximation for the gradient to accelerate the dual-coordinate descent method. \citet{chan2023plame} develop a similarity measure PLAME (piecewise-linear approximate measure) that is incorporated with the dual-coordinate descent method. By this they approximate the additive kernel to ensure an efficient training of the additive kernel SVM. Furthermore, additive kernels are employed when prior knowledge about the distribution of the data is given or an easily interpretable prediction function is desired~\citep{christmann2012consistency}. \citet{xie2018uncertain} suggest the UKSVM (uncertain kernel SVM) method for classifying uncertain data. Another relevant application of additive kernel SVMs is histogram-based image comparison. Common choices for such additive histogram comparison kernels are the intersection or chi-squared kernel. Linearly combining functions of each coordinate of the histogram yields the comparison~\citep{maji2012efficient}.

Additive kernels are also used in GP models. \citet{durrande2012additive} argue that even if the function to be approximated is not purely additive, additive Kriging models can express the additivity of the function well. They are combining the features of GP modeling with generalized additive models what is especially suitable for high-dimensional problems. A similar approach is presented by~\citet{durrande2011additive}, where additivity is incorporated in the covariance kernel to obtain GPs for additive models. The response of the generalized additive models (GAMs)~\citep{hastie1986generalized} simulator can then approximately be separated into a sum of univariate functions. \citet{duvenaud2011additive} propose an expressive but tractable parameterization of the kernel function. By this all input interaction terms can be evaluated efficiently. The additive structure incorporated into the model which is present in many real data sets leads to increased interpretability and predictive power what yields a better performance compared to standard GP models overall.

% Furthermore, using additive kernels can also originate from dealing with source separation tasks where a signal shall be separated into additive components. In~\citep{liutkus2014kernel} the authors suggest to apply a kernel additive modeling based on local regression that allows for efficient separation of multidimensional signals. Defining nearness by a source-specific proximity kernel they estimate a source at a location by its values at other locations nearby.

In many papers using a kernel of the form
\begin{align} \label{tradANOVAkernel}
    \kappa_{\text{A}}^t (\bm{x}, \bm{x}') = \sigma_f^2\sum_{1 \le i_1 < i_2 < \dots < i_t \le d} \prod_{j=1}^t \kappa_{i_j} (x_{i_j}, x_{i_j}'),
\end{align}
with $\bm{x}, \bm{x}' \in \mathbb{R}^d$, $t\leq d$ and the signal variance parameter $\sigma_f \in \mathbb{R}_{>0}$, is suggested. It is referred to as the $t$-th order additive kernel~\citep{duvenaud2011additive} or the ANOVA kernel~\citep{shawe2004kernel}. The term ANOVA is an abbreviation for analysis of variance. This is reasonable in this setting since we aim to work with a kernel that compares the variance of the data in detail. The kernels $\kappa_{i_j}$ are one-dimensional base kernels. Depending on the order $t$ of the additive kernel multiple base kernels are multiplied to cover higher order feature interactions. Overall $\kappa_{\text{A}}^t$ yields the sum of all possible $t$-th order interactions of one-dimensional base kernels. Here, the idea is to compare the data on a subset of features first and summing over several of such kernels relying on fewer features~\citep{shawe2004kernel}.

Assume all base kernels are squared-exponential kernels
\begin{align*}
    \kappa_{\text{SE}} ( \bm{x}, \bm{x}') = \exp \left( - \frac{\| \bm{x} - \bm{x}' \|_2^2}{2 \ell^2} \right)
\end{align*}
with the same length-scale parameter $\ell \in \mathbb{R}_{ >0}$. Then, the $d$-th order ANOVA kernel is nothing else than the squared-exponential kernel evaluated at all feature dimensions at once, since
\begin{equation*}
    \begin{aligned}
        \kappa_{\text{A}}^d ( \bm{x}, \bm{x}') &= \sigma_f^2 \prod_{j=1}^d \kappa_j (x_j, x_j')
        = \sigma_f^2 \prod_{j=1}^d \exp \left( - \frac{\|x_j - x_j' \|_2^2}{2 \ell^2} \right) \\
        &= \sigma_f^2 \exp \left( - \sum_{j=1}^d \frac{ \|x_j - x_j'\|_2^2}{2 \ell^2} \right)
        = \sigma_f^2 \exp \left( - \frac{ \|\bm{x} - \bm{x}'\|_2^2}{2 \ell^2} \right) \\
        &= \sigma_f^2 \kappa_{\text{SE}} ( \bm{x}, \bm{x}')
    \end{aligned}
\end{equation*}
by the power law. The sum of all such $t$-th order ANOVA kernels is the full additive kernel
\begin{align} \label{full_add_kernel}
    \kappa_{\text{A}}^{\text{full}} (\bm{x}, \bm{x}') = \sum_{t=1}^d \kappa_{\text{A}}^t (\bm{x}, \bm{x}').
\end{align}
However, involving all subsets can be too over-determined for a good approximation of the data structure. Instead combining  terms only relying on a small number of features can be more promising~\citep{stitson1999support}.

In this paper, we propose the use of a special case of the ANOVA kernel~\eqref{tradANOVAkernel}, whose design is discussed in detail in the following.

It is a common phenomenon that many real-world data sets are mainly based on sums of low-order feature interactions~\citep{duvenaud2011additive}. Therefore, we do not want to incorporate the full additive kernel~\eqref{full_add_kernel} merging base kernels of all possible dimensionality. At the same time, we need to ensure to capture non-local structure, which is why we do not incorporate all feature dimensions within one kernel evaluation at once. Hence, we suggest to work with a weighted sum of kernels
\begin{align} \label{eq:wind_gauss_anova_kernel}
    \kappa ( \bm{x}_i , \bm{x}_j ) = \sigma_{f}^2 \sum_{s=1}^P \kappa_s ( \bm{x}_i , \bm{x}_j ),
\end{align}
where $\bm{x}_i , \bm{x}_j \in \mathbb{R}^d$, $i,j = 1, \dots, N$, and the sub-kernels $\kappa_s$ depend on low-dimensional feature interactions only. For this we define sets of feature indices $\mathcal{W}_s$ building the $s$-th group of features, that is
\begin{align*}
    \kappa_s (\bm{x}_i, \bm{x}_j) = \exp \left( - \frac{\| \bm{x}_i^{\mathcal{W}_s} - \bm{x}_j^{\mathcal{W}_s} \|_2^2}{2 \ell^2} \right),
\end{align*}
for $s = 1, \dots, P$ and $P \in \mathbb{N}$ the number of feature groups and sub-kernels. For the bivariate case the groups of features $\mathcal{W}_s \in \{ (a,b): a,b \in \{1, \dots, d\}, a \neq b\}$ are of length $d_s=2$ and for the trivariate $\mathcal{W}_s \in \{ (a,b,c): a,b,c \in \{1, \dots, d\}, a \neq b \neq c\}$ with $d_s=3$ for instance.

Accordingly, $\bm{x}_i^{\mathcal{W}_s} \in \mathbb{R}^{d_s}$, $i = 1, \dots, N$, are the corresponding data points restricted to those indices. The resulting kernel
represents the special case of the ANOVA kernel~\eqref{tradANOVAkernel} with $t = |\mathcal{W}_s|= d_s$ and Gaussian base kernels relying on $d_s$-dimensional windows $\mathcal{W}_s$ of features, and is referred to as the additive Gaussian kernel from now on. We aim for a method that keeps the computational complexity for large-scale applications low. Since more feature windows lead to more kernels what leads to solving more linear systems with dense matrices that are square in the number of data points, we cannot work with all possible low order feature interactions. Instead we require a procedure for reasonably reducing the number of windows.

Strategies on how to determine the sets of feature indices $\mathcal{W}_s$ are elaborated in Section~\ref{Sec:Feature engineering techniques}. If the number of data points $N$ is large, multiplying and solving with $K = \sigma_f^2 \sum_{s=1}^P K_s$ is of quadratic or even cubic computational complexity. For this, we suggest employing NFFT-accelerated approximations in Section~\ref{Sec: NFFT}. Naturally, the choice of the parameter values can have a huge impact on the prediction quality and the parameters have to be optimized. For this, we introduce a NFFT-acceleration procedure for multiplying and solving with the matrix representing the derivative of the kernel with respect to the length-scale hyperparameter $\ell$, that is $K_{\text{der}} = \sigma_f^2 \sum_{s=1}^P K_s^{\ell}$ with
\begin{align*}
    \kappa_s^{\ell} (\bm{x}_i, \bm{x}_j) = \frac{ \| \bm{x}_i^{\mathcal{W}_s} - \bm{x}_j^{\mathcal{W}_s} \|_2^2}{\ell^3} \exp \left( - \frac{\| \bm{x}_i^{\mathcal{W}_s} - \bm{x}_j^{\mathcal{W}_s} \|_2^2}{2 \ell^2} \right).
\end{align*}
We demonstrate the corresponding approximation error empirically and provide error estimates analytically. In Section~\ref{Sec: Global sensitivity analysis} we present an approach on how to determine $\mathcal{W}_s$ analytically in the Fourier setting and run first experiments according to this scheme. We showcase the numerical performance of the presented feature grouping techniques for the additive kernel design with NFFT-accelerated kernel evaluations in Section~\ref{Sec:Numerical results} and present concluding remarks in Section~\ref{Sec:Conclusion}.

% Subsection: Contribution
\subsection*{Main Contributions}

We summarize the main contributions of this paper as follows:

\begin{itemize}
    \item Novel combination of feature grouping techniques for the additive kernel setting with NFFT-accelerated kernel evaluation.
    \item Development of an NFFT-acceleration procedure for the derivative kernel.
    \item Derivation of Fourier error estimates for the trivariate Gaussian kernel and its derivative kernel.
\end{itemize}

% Section: Feature Engineering Techniques
\section{Feature Engineering Techniques} \label{Sec:Feature engineering techniques}

In this section we focus on determining the feature arrangement for the additive kernel introduced previously. We give a thorough overview of existing feature engineering methods in the literature. We later choose a few methods that are most suitable for the additive kernel setting. For the sake of comparability and reproducibility we focus on open source Python software. The performance of those methods is analyzed and compared comprehensively in Section~\ref{Sec:Numerical results}.

% Subsection: Feature Selection and Grouping Techniques in the Literature
\subsection{Feature Selection and Grouping Techniques in the Literature} \label{Sec:Feature selection and grouping techniques in the literature}

We now give an overview of feature selection and feature grouping techniques proposed in the literature. This shall form the basis for feature arrangement techniques in the additive kernel setting.

% Subsubsection: Feature Selection
\subsubsection{Feature Selection}

In the literature many definitions for feature selection can be found. It is described as a dimensionality reduction technique that chooses a smaller subset of features from the original feature set trying to fulfill different criteria. In general, the most common objectives of feature selection are improving the prediction quality, reducing the computation (training and utilization) time, reducing measurement and storage requirements, increasing the comprehensibility and interpretability, gaining a better understanding of the data and stemming the curse of dimensionality~\citep{chandrashekar2014survey,li2017feature,kumar2014feature,guyon2003introduction,dash1997feature,miao2016survey}.

The most common approach to categorize feature selection algorithms is classifying them into filter, wrapper and embedded methods~\citep{kumar2014feature,chandrashekar2014survey,jovic2015review,guyon2003introduction,miao2016survey,venkatesh2019review,saeys2007review}. Filter methods select the features based on their importance with respect to the target concept~\citep{bolon2013review}. For this statistical measures such as information gain~\citep{yang1997comparative,raileanu2004theoretical}, chi-square test~\citep{yang1997comparative}, Fisher score, correlation coefficient, variance threshold, reliefF~\citep{kira1992practical,raileanu2004theoretical}, F-statistic~\citep{ding2005minimum} or mRMR~\citep{peng2005feature} are employed~\citep{venkatesh2019review,miao2016survey}. Filter methods are independent of the choice of the learning algorithm. Therefore the obtained feature subsets can subsequently be transferred to any learning task on that data set. This is not the case for wrapper methods where feature selection is accomplished by performing the learning task on candidate subsets until a stopping criterion is met~\citep{kohavi1997wrappers}. By design, wrapper methods are more computationally complex than filter methods but often more accurate~\citep{venkatesh2019review}. Examples are recursive feature elimination~\citep{yan2015feature}, the sequential feature selection algorithm~\citep{jain1997feature} or genetic algorithms. In embedded methods the feature selection is performed as part of the learning process with a specific learning algorithm~\citep{kumar2014feature}. One example for embedded methods are random forests~\citep{venkatesh2019review}. Embedded methods are based on the same idea as wrapper methods but they are working with an objective function consisting of a goodness-of-fit term and a penalty for large number of variables~\citep{guyon2003introduction} what makes them more efficient.

Often, the general procedure for feature selection is characterized as $4$ key steps: subset generation, evaluation of the subset, stopping criteria and validation of the result~\citep{kumar2014feature,liu2005toward,jovic2015review,dash1997feature}. \citet{dash1997feature} furthermore define $3$ categories of generation procedures (complete, heuristic, random) and $5$ categories of evaluation functions (distance, information, dependence~\citep{song2012feature}, consistency, classifier error rate).

Other common categorization approaches found in the literature focus on the availability of the label information, the data perspective or the label and selection strategy perspective. \citet{miao2016survey} distinguish between supervised, semi-supervised and unsupervised methods~\citep{dy2004feature}. \citet{yu2004efficient} present feature selection techniques based on relevance and redundancy.

In general, when looking for a suitable feature selection method the following aspects should be considered: simplicity, stability, the desired number of reduced features, the classification accuracy and storage and computational requirements~\citep{chandrashekar2014survey}.

% Subsubsection: Feature Grouping
\subsubsection{Feature Grouping}

In addition to classical feature selection approaches as described above several feature selection techniques based on previous feature grouping can be found in the literature. The main idea behind this approach is to generate groups of features where the intra-group similarity is maximized and the inter-group similarity is minimized~\citep{kuzudisli2023review}. Afterwards feature selection is performed based on this group arrangement. 

This procedure is motivated by discovering that relevant features are highly correlated in a high-dimensional setting. Therefore, groups of correlated features resistant to sample size variations can be formed~\citep{garcia2016high}.

To that effect, feature grouping comes with many benefits when learning with high-dimensional data~\citep{garcia2016high}. The search space being reduced with feature grouping and the higher resistance to sample variations~\citep{kuzudisli2023review} leads to an improved stability of feature selection~\citep{jornsten2003simultaneous} and effectiveness of the search~\citep{garcia2016high}, helps to reduce the complexity of the model and increase the generalization capability~\citep{kuzudisli2023review} and potentially reduces the estimator variance~\citep{shen2010grouping}. Popular applications are text mining~\citep{bekkerman2003distributional,dhillon2003divisive} or microarray domains~\citep{ben1999clustering}.

Among the most basic feature grouping methods are exhaustive or explicit search for feature groups~\citep{zheng2021feature,zhong2012efficient}. However, this combinatorial optimization problem is often computationally infeasible when working with large data sets. To overcome this, greedy hill climbing strategies have been proposed, with features leading to the largest gain in the subset score greedily being added to the candidate subset individually~\citep{zheng2021feature}. However, they typically only lead to local optima.

Regularization techniques are very common algorithms for generating feature groups and belong to the category of embedded methods~\citep{garcia2016high,zhong2012efficient,yang2012feature,han2015discriminative}. By adding a regularization term to the objective function the model is fitted minimizing the coefficients what results in features with coefficients close to zero being dropped~\citep{garcia2016high}. Common sparse-modeling algorithms are the lasso regularization~\citep{tibshirani1996regression} and its extensions such as group lasso~\citep{yuan2006model}, adaptive lasso, fused lasso and clustered lasso. Further regularization techniques worth mentioning are Bridge regularization~\citep{huang2008asymptotic}, elastic net regularization~\citep{zou2005regularization} and the orthogonal shrinkage and clustering algorithm for regression (OSCAR) method introduced by \citet{bondell2008simultaneous}. However, many of the aforementioned methods suffer from the problem that they cannot distinguish groups of features that are similar but still different well and often tend to merge those groups together.

As a remedy, discriminative feature grouping (DFG) is proposed by~\citet{han2015discriminative}. By introducing a novel regularizer in the feature coefficients fusing and discriminating feature groups is balanced out. Moreover, they present an adaptive DFG (ADFG) aiming to yield a better asymptotic property.

Subspace clustering methods represent another type of feature grouping techniques that intend to detect clusters in subspaces rather than the whole data space \citep{chen2012feature}. They are distinguished between hard subspace clustering~\citep{agrawal1998automatic} and soft subspace clustering methods~\citep{huang2005automated,domeniconi2007locally,jing2007entropy}. While hard subspace clustering methods detect the exact subspaces of the clusters, in soft clustering methods subspaces with large weights are identified by assigning weights to features instead~\citep{chen2012feature,gan2015subspace}. Many of such methods have been proposed in the literature, such as CLIQUE~\citep{agrawal1998automatic}, W-$k$-means~\citep{huang2005automated}, fuzzy subspace clustering (FSC)~\citep{gan2008convergence}, EWKM~\citep{jing2007entropy}, LAC~\citep{domeniconi2007locally} and EEW-SC~\citep{deng2010enhanced}. Alternatives that are less sensitive to noise and missing values~\citep{chen2012feature,gan2015subspace} are FG-$k$-means, introduced by~\citet{chen2012feature} as an iterative alternative soft subspace clustering method, and AFG-$k$-means~\citep{gan2015subspace}. While the feature groups are assumed to be given as inputs in FG-$k$-means, the groups are detected automatically by dynamically updating them during the iterative process in AFG-$k$-means instead.

Regarding stability, several group-based feature selection methods were developed for the purpose of improving robustness. The main reason for instability in feature selection techniques originates in the small number of samples in a high-dimensional domain and the goal of selecting the minimal subset without redundant features~\citep{yu2008stable}. For this two frameworks have been proposed: dense feature groups (DFG)~\citep{yu2008stable} and consensus group stable feature selection (CGS)~\citep{loscalzo2009consensus}. Overall, the concept of those methods originates from the observation that features close to core (peak) regions have a high correlation.

% Subsection: Feature Arrangement Techniques for Additive Kernels
\subsection{Feature Arrangement Techniques for Additive Kernels} \label{Sec:Feature arrangement techniques for additive kernels}

After having presented existing feature engineering techniques above, we want to choose methods suitable for arranging the features in the additive kernel setting. In this context, we refer to feature grouping as separating feature dimensions into multiple kernels by defining corresponding feature windows $\mathcal{W}_s$ as introduced in \eqref{eq:wind_gauss_anova_kernel}. For this, several requirements have to be fulfilled. 

First, we do not only want to get rid of less relevant features but also need a sensible scheme for separating the feature indices into several small groups. Additionally, we want to keep the number of kernels $P$ small since more kernels lead to higher computational costs as more matrix vector products need to be evaluated. The kernel matrix vector product approximations are more expensive the larger the size of the corresponding feature subset $d_s$. In order to exploit the full computational power of those approximation techniques $d_s$ is required to be small. Since both demands are opposed to each other, the number $P$ of kernels or the number of feature groups respectively needs to be balanced carefully with the cardinality of the feature groups.

Second, we do not necessarily have the kernel entries given explicitly. When working with large-scale problems fast approximation methods are often employed for speeding up multiplications with the kernel matrices. Then, the routine operates as a black box, where the data points and a vector are given as inputs and the kernel vector product is returned as the output. We go into more detail in Section~\ref{Sec: NFFT}. Indeed, a number of feature selection techniques require having those entries available explicitly.

In the remainder of this section we discuss several feature grouping techniques that aim to determine the feature groups in a sophisticated way. We describe some very basic feature grouping strategies first before we consider more elaborated ones next. In Section~\ref{Sec:Numerical results} we analyze and compare their performance.

Note that we refrain from adding certain feature grouping methods to our investigations even though they are somewhat prevalent in the literature. Examples are OSCAR, CLIQUE, FGOC (feature grouping and orthogonal constraints) that are not competitive regarding their computational complexity. Hierarchical clustering, (adaptive) discriminative feature grouping, k-means and fuzzy c-means clustering are methods where the strategy on how to define feature windows are not suited to the setting we want to employ.

% Subsubsection: Straightforward Feature Grouping Methods
\subsubsection{Straightforward Feature Grouping Methods} \label{Sec:Straight-forward feature grouping methods}

A very basic strategy on how to separate the feature dimensions is to simply group the features following their feature indices determined by the column arrangement in the original data set. For $d=6$ this yields windows $\mathcal{W}_1 = \{ 1,2\}$, $\mathcal{W}_2 = \{ 3,4\}$, $\mathcal{W}_3 = \{ 5,6\}$ for $t=d_s=2$ and windows $\mathcal{W}_1 = \{ 1,2,3\}$, $\mathcal{W}_2 = \{ 4,5,6\}$ for $t=d_s=3$ respectively, for instance. For $t=d_s=1$ this represents a special case of the feature index based allocation. Then, the feature dimensions are split into $d$ one--dimensional windows and the features are arranged as $\mathcal{W}_s=\{s\}$ for $s = 1, \dots, d$.

Even though this strategy is very basic it constitutes a valuable comparative method for examining whether putting more effort in terms of computational complexity and runtime into more complex techniques pays off in achieving higher predictive accuracy. 

% Subsubsection: Methods Based on Feature Importance Ranking 
\subsubsection{Methods Based on Feature Importance Ranking}

In previous \\works \citep{nestlerlearning,wagner2023preconditioned} we ranked all features by their mutual information score (MIS)~\citep{battiti1994using} and arranged them into groups of $3$ following their importance scores starting with the largest one. The MIS quantifies how much information about the label can be obtained by knowing the feature value. It is a univariate measure that does not examine the impact of a combination of several features on predicting the target. However, the MIS ranking method only requires the original data as input and is of low computational complexity. After having employed this feature arrangement technique in previous papers already we now want to analyze its performance more extensively by comparing it to several other methods. 

Instead of computing the importance scores via MIS, other measures can be applied. Common alternatives are the Fisher score and reliefF. Alternatively, feature importance can be obtained by fitting a decision tree model. By introducing a threshold, features with an importance score below this value are dropped.

Based on the feature importance scores, the features are ranked and arranged into groups of the desired size. The feature arrangement can follow different strategies. Features can either be arranged consecutively following the ranking such that the features with the $3$ highest scores are arranged into the first window and so on for $| \mathcal{W}_s | = 3$ for instance or the features are separated into different feature groups successively so that features with similar importance scores do not end up in the same group. We refer to these arrangement strategies as `consec' and `distr' from now on. 

% Subsubsection: Regularization Techniques
\subsubsection{Regularization Techniques} 

In contrast to computing each feature's importance score individually as described above, one can work with a regression model for estimating sparse coefficients. In the well-known Lasso regularization, the objective function
\begin{align*}
    Z_{\text{Lasso}} = \frac{1}{2N} \| Xw - y \|_2^2 + \lambda_{\text{Lasso}} \|w \|_1
\end{align*}
is minimized with respect to the coefficients $w$ and $\lambda_{\text{Lasso}} > 0$, the regularization parameter that regulates the degree of sparsity of the  estimated coefficient vector.

Elastic-Net is a regression model incorporating both the $\text{L}1$-norm and the $\text{L}2$-norm of the coefficients to the model. The corresponding objective
\begin{align*}
    Z_{\text{EN}} = \frac{1}{2N} \| Xw - y \|_2^2 + \lambda_{\text{EN}} \rho \|w \|_1 + \frac{\lambda_{\text{EN}} (1-\rho)}{2} \|w \|_2
\end{align*}
is again minimized with respect to the coefficients $w$ and the ratio between the penalty terms is balanced with the $\text{L}1$-ratio $\rho$. Note that with $\rho=1$ the objective of Elastic-Net equals with Lasso's objective.

By combining $\text{L}1$ and $\text{L}2$ regularization, Elastic-Net benefits from both the sparsity of the Lasso model and the regularization properties of ridge, such as stability. However, in settings with two correlated features, Lasso randomly chooses one of them while Elastic-Net encourages a grouping effect and tends to select both~\citep{zou2005regularization}.

Note that the features are selected based on classical regression on the data matrix $X \in \mathbb{R}^{N \times d}$ rather than on how they perform in a non-linear context and hence we are working in a different context than in the kernel setting here.

In addition to `consec' and `distr' we introduce the `direct' feature arrangement strategy for Lasso and Elastic-Net. In `direct' the features with nonzero coefficients are immediately assigned to the windows consecutively following their indices without ranking them first.

% Subsubsection: Feature Arrangement Based on Clustering
\subsubsection{Feature Arrangement Based on Clustering} 

Another approach for detecting feature groups is via feature clustering techniques. One way of doing this is via connected components. This method is based on the correlation matrix holding the Pearson product-moment correlation coefficients. In clustering via connected components two features are considered to be connected if the magnitude of their correlation value is larger than some predefined threshold. Based on those pairs the feature clusters are detected.

Different to the feature importance ranking and regularization techniques described above, clustering methods are unsupervised and do not incorporate the target values into the clustering process.

In addition to `consec' and `distr' we introduce `single' as a feature arrangement strategy for the connected components method. In `single' all centroid features build a window of length $1$.

% Subsubsection: Feature Grouping Optimization via Regularization
\subsubsection{Feature Grouping Optimization via Regularization} 

Alternatively to the aforementioned approaches, we propose to determine the feature groups via an optimization. The objective $Z_{\text{fg}}$ of this feature grouping optimization consists of the objective $Z$ of the original classification/regression method plus a $\text{L}1$ regularization term, that is
\begin{align*}
    Z_{\text{fg}} ( \theta) = Z(\theta) + \lambda_{\text{fg}} \| \sigma_f^{\text{fg}} \|_1,
\end{align*}
with $\lambda_{\text{fg}} > 0$ the regularization parameter balancing the impact of the $\text{L}1$ penalty,
\begin{align*}
    K_{\text{fg}} = \underbrace{\sigma_{f_1}^2 K_1}_{K_1^{\text{fg}}} + \dots + \underbrace{\sigma_{f_P}^2 K_P}_{K_P^{\text{fg}}}
\end{align*}
and $\theta_{\text{fg}} = (\sigma_f^{\text{fg}}, \ell, \sigma_{\varepsilon})$, where $\sigma_{\varepsilon}$ is the noise parameter and $\sigma_f^{\text{fg}} = [\sigma_{f_1}, \dots, \sigma_{f_P}]^{\intercal}$. Note that in contrast to the definition of the additive Gaussian kernel $\kappa$ in \eqref{eq:wind_gauss_anova_kernel}, all sub-kernels $K_s^{\text{fg}}$ are assigned a separate kernel weight $\sigma_{f_s}$ now, with respect to which the optimization is performed. The noise $\sigma_{\varepsilon}$ and length-scale $\ell$ parameters are kept fixed during the feature grouping optimization.

The model is initialized with all possible feature subsets with $d_s=2$, what leads to $P = {d \choose 2}$ sub-kernels $K_s^{\text{fg}}$. Through the $\text{L}1$ regularization term sparsity is ensured and most of the kernel weights $\sigma_{f_s}$ are pushed to zero. By this, only a few non-zero kernel weights are obtained, what yields the desired optimal feature groups immediately.

Since the binomial coefficient grows big even for moderate feature dimensions, the feature grouping optimization is performed on a small subset of the data set only. Additionally, one can perform a feature ranking initially, using the MIS ranking for instance, to reduce the number of pairs $P$ by dropping the features least relevant in the very beginning. While the feature importance ranking based and clustering techniques do not allow for repeated feature indices, this feature grouping optimization approach enables feature indices to appear in multiple feature windows $\mathcal{W}_s$.

%% Section: NFFT-Accelerated Kernel Vector Products and Multiplications with the Derivative Kernel
\section[NFFT-Accelerated Kernel Vector Products and Multiplications with the Derivative Kernel]{\texorpdfstring{NFFT-Accelerated Kernel Vector Products and Multiplications\\ with the Derivative Kernel}} \label{Sec: NFFT}

When working with large-scale data, multiplying and solving with the dense kernel matrix is the classical computational bottleneck. In this section we give an overview of techniques for accelerating evaluations with the kernel matrix and its derivative. 

Examples for such methods are structured kernel interpolation (SKI), subset of regressors (SoR), deterministic training conditional (DTC), fully independent training conditional (FITC) and partially independent training conditional (PITC) approximation or hierarchical matrices (H-matrices). SKI is an approach based on inducing points that accelerates kernel approximations through kernel interpolation~\citep{wilson2015kernel}. Another inducing point approach is SoR that approximates kernel vector multiplications based on inducing points with a specific prior for the vector. The DTC approximation works similarly to the SoR except from the relation between the function value and the inducing points being described by an exact test conditional instead of being deterministic such as for SoR. This means that with DTC the predictive response has a prior variance of its own~\citep{quinonero2005unifying}. FITC is another likelihood approximation with an extensive covariance. Different to SoR and DTC, FITC does not introduce a deterministic relation between the function value and the inducing points. For this, it employs an approximation to the training conditional distribution as an independence assumption~\citep{quinonero2005unifying}. PITC further improves this approximation by equipping the training conditional with a block diagonal covariance~\citep{quinonero2005unifying}. An alternative approach are hierarchical matrices that are data-sparse approximations of non-sparse matrices by partitioning them into low-rank factorized sub-blocks~\citep{borm2003introduction}.

In this paper we emphasize the NFFT-based fast summation technique we employed in \citet{nestlerlearning} and \citet{wagner2023preconditioned}. In fact, all the above-mentioned fast kernel vector product approximation techniques have in common that their effectiveness and high efficiency is restricted to small feature space dimensions. This again motivates the need to splitting the feature space and working with additive kernels. 

Moreover, we demonstrate the effect of hyperparameter choices on learning tasks and highlight the importance of hyperparameter optimization. Naturally, for this the derivatives with respect to the hyperparameters are required and one typically is faced with the task of multiplying also with the derivative matrix for that particular hyperparameter. We here advocate for an explicit computation employing the kernel structure as much as possible as finite difference approximations are typically not stable enough. Another alternative would be automatic differentiation techniques such as the one employed in \citet{charlier2021kernel}.
We introduce an NFFT-based technique for approximating multiplications with derivative kernels. A typical example would be the parameter training for Gaussian process regression where due to the log-determinant one typically requires matrix-vector products with the derivative matrix as part of a matrix function approximation for the correct computation of the parameter gradient.

% Subsection: Fourier Theory
\subsection{Fourier Theory}
In many kernel learning tasks such as the GP hyperparameter  optimization, multiplying with the kernel matrix $K \in \mathbb{R}^{N, N}$ is most computationally complex. The cost of computing its product $K v$ with a vector $v\in\R^N$ through the conjugate gradient method is $\mathcal{O} (N^2)$ in each iteration, for instance. This scales badly for large-scale applications. Therefore, we approximate these products leveraging the computational power of the non-equispaced fast Fourier transform (NFFT) instead. This is realized by applying the fast summation approach, in which the NFFT and the adjoint NFFT, confer~\citet{potts2003fast}, are combined to compute sums of the form
\begin{align} \label{eq: fast sum}
    h( \bm{x}_i' ) \coloneqq \sum_{j=1}^N v_j \kappa ( \bm{x}_i', \bm{x}_j ) \quad \forall i = 1, \dots, N'
\end{align}
efficiently. For this, the kernel $\kappa$ is approximated by a trigonometric polynomial, what can be written as
\begin{align} \label{Fourier representation}
    \kappa (\x',\x) = \kappa (\br) \approx \sum_{\bm k\in\mathcal{I}_{\bm{M}}} \hat{c}_{\bk} \e^{2 \pi \mathrm{i} \bk^\intercal \br / L},
\end{align}
where $\x,\x'\in\R^d$, $\br \coloneqq \x' - \x$, $L$ is the period that has to be chosen appropriately, $\hat{c}_{\bk}$ are the Fourier coefficients and $\bm{M} = ( M_1, \dots, M_d )^\intercal \in 2 \mathbb{N}^d$ is a multivariate grid size, that is a $d$-dimensional vector with even integer components, what gives multivariate index sets
\begin{align*}
    \mathcal{I}_{\bm{M}} \coloneqq \{ - \textstyle{\frac{M_1}{2}}, \dots, \textstyle{\frac{M_1}{2}} -1 \} \times \dots \times \{ - \textstyle{\frac{M_d}{2}}, \dots, \textstyle{\frac{M_d}{2}} -1 \}
\end{align*}
of cardinality $|\mathcal I_{\bm M}|=M_1\cdot\ldots\cdot M_d$.
Typically, we set $M_1 = \dots = M_d = m$, so that the grid size is the same respective all dimensions.
Replacing $\kappa$ in \eqref{eq: fast sum} by its Fourier representation \eqref{Fourier representation} and rearranging the sums gives
\begin{align} \label{eq: approx fast sum}
    h ( \bm{x}_i' )
    \approx  \tilde{h} ( \bm{x}_i' ) = \sum_{j=1}^{N} v_j \sum_{\bm{k} \in \mathcal{I}_{\bm{M}}} \hat{c}_{\bm{k}} \e^{2 \pi \mathrm{i} \bm{k}^ \intercal ( \tilde{\bm{x}}_i' - \tilde{\bm{x}}_j )}
    = \sum_{\bm{k} \in \mathcal{I}_{\bm{M}}} \hat{c}_{\bm{k}} \left(\sum_{j=1}^{N} v_j \e^{-2 \pi \mathrm{i} \bm{k}^ \intercal \tilde{\bm{x}}_j} \right) \e^{2 \pi \mathrm{i} \bm{k}^ \intercal \tilde{\bm{x}}_i'},
\end{align}
where $\tilde{\bm{x}}_i'$ and $\tilde{\bm{x}}_j$ are now scaled nodes, for which $\tilde{\bm{x}}_i'-\tilde{\bm{x}}_j\in[-1/2,1/2]^d$ holds.
Note that $\tilde h$ is now a periodic function with period $1$ in each coordinate direction.
By this, we can reduce the arithmetic complexity for computing $Kv$ from $\mathcal{O} (N^2)$ to $\mathcal{O} (N \log N)$, provided that the parameters involved are chosen appropriately.
The inner sums for all $\bm{k}$ can be computed via a so-called adjoint NFFT (or type-2 nonuniform FFT) and the approximation of the outer sums are then realized by a $d$-variate NFFT (or type-1 nonuniform FFT). This procedure is known as NFFT-based fast summation.
NFFT and adjoint NFFT themselves are approximate algorithms for an efficient evaluation of the required trigonometric sums at equidistant nodes. The accuracy of these algorithms is controlled by several parameters, which we do not further discuss here.
For detailed information concerning NFFT and related algorithms we refer to~\citet{NFFTrepo} and references therein.

In our investigations, the kernel function $\kappa$ is non-periodic. Thus, the approximation by a trigonometric polynomial is not straightforward. We refer to our previous paper~\citep{nestlerlearning} and references therein for more details on the underlying theory.

In 1D, the easiest periodization approach just continues the kernel function periodically in order to obtain a continuous $1$-periodic function $\tilde \kappa(r)\coloneqq \kappa(r+k)$, where $k\in\Z$ is chosen such that $r+k\in[-1/2,1/2]$, see Figure~\ref{fig:periodization} for an illustration. The Fourier coefficients of that $C_0(\T)$-continuation will tend to zero like $\mathcal O(k^{-2})$.
A faster decay of the Fourier coefficients can be achieved by a smoother periodization, in which the function is regularized at the edges by a smooth transition, see \citet{potts2003fast}.

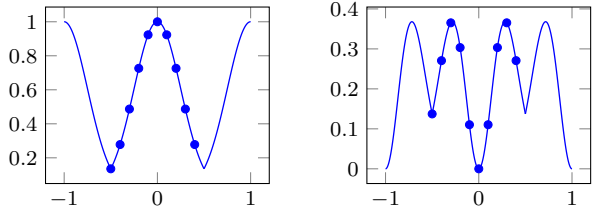
\begin{figure}[ht]
\centering
\begin{tikzpicture}\footnotesize
        % SE kernel
        \begin{axis}
        [width=0.35\linewidth]
        \addplot[color=blue,solid,line width=0.5pt,domain=-0.5:0.5,samples=300]{exp(-(x^2/(2*0.25^2))))};
        \addplot[color=blue,solid,mark=*,/tikz/mark size=1.5,draw=none,domain=-0.5:0.4,samples=10]{exp(-(x^2/(2*0.25^2))))};
        \addplot[color=blue,solid,line width=0.5pt,domain=-1:-0.5,samples=150]{exp(-((x+1)^2/(2*0.25^2))))};
        \addplot[color=blue,solid,line width=0.5pt,domain=0.5:1,samples=150]{exp(-((x-1)^2/(2*0.25^2))))};
        \end{axis}
\end{tikzpicture}\qquad
\begin{tikzpicture}\footnotesize
        % "derivative" kernel
        \begin{axis}
        [width=0.35\linewidth]
        \addplot[color=blue,solid,line width=0.5pt,domain=-0.5:0.5,samples=300]{exp(-(x^2/(2*0.2^2))))*x^2/(2*0.2^2)};
        \addplot[color=blue,solid,mark=*,/tikz/mark size=1.5,draw=none,domain=-0.5:0.4,samples=10]{exp(-(x^2/(2*0.2^2))))*x^2/(2*0.2^2)};
        \addplot[color=blue,solid,line width=0.5pt,domain=-1:-0.5,samples=150]{exp(-((x+1)^2/(2*0.2^2))))*(x+1)^2/(2*0.2^2)};
        \addplot[color=blue,solid,line width=0.5pt,domain=0.5:1,samples=150]{exp(-((x-1)^2/(2*0.2^2))))*(x-1)^2/(2*0.2^2)};
        \end{axis}
\end{tikzpicture}
\caption{An even function of the form $f(r)=\mathrm{exp}(-r^2/2\ell^2)$ (left) or $f(r)=\mathrm{exp}(-r^2/(2\ell^2))\cdot(r^2/2\ell^2)$ (right), defined on $[-1/2,1/2]$, is periodized via simple periodic continuation.
The resulting periodic function is at least continuous, but in general not smooth. A finite number of approximating Fourier coefficients can be obtained by sampling the function in equidistant points (marked by the dots) and applying the FFT.
Alternatively, one can make use of the analytic Fourier coefficients, provided they are known.
\label{fig:periodization}}
\end{figure}

The presented periodization technique is also applicable to multivariate radial kernels in order to periodize the function and approximate it for $\br\in[-1/2,1/2]^d$.
In the numerical experiments we make use of the NFFT-based fast summation approach \citep{NFFTrepo}, where only a radial section of the function is approximated, that is $\|\br\|\leq\tfrac12$. For details we refer to~\citet{postni04}.

For the Gaussian kernel $\kappa_\text{gauss}(r) \coloneqq \e^{-r^2/2\ell^2}$ we can compute the Fourier coefficients of $\tilde\kappa_\text{gauss}(r)$ as
\begin{align*}
    c_k(\tilde \kappa_\text{gauss})
    & = \int_{-1/2}^{1/2} \e^{-\frac{r^2}{2\ell^2}} \e^{2\pi\mathrm{i} kr} \mathrm dr \\
    & = \e^{-2\pi^2k^2\ell^2} \int_{-1/2}^{1/2} \e^{-(\frac r{\sqrt2\ell}-\pi\mathrm{i} k\sqrt2\ell)^2} \mathrm dr \\
    %& = \frac 12 \int_{-1/2}^{1/2} \e^{-\frac{r^2}{2\ell^2}} \e^{2\pi\mathrm{i} kr} \mathrm dr + \frac 12 \int_{-1/2}^{1/2} \e^{-\frac{r^2}{2\ell^2}} \e^{-2\pi\mathrm{i} kr} \mathrm dr \\
    %& = \frac{\e^{-2\pi^2 k^2\ell^2}}{2} \int_{-1/2}^{1/2} \e^{-(\frac r{\sqrt2\ell}-\pi\mathrm{i} k\sqrt2\ell)^2} \mathrm dr + \frac{\e^{-2\pi^2 k^2\ell^2}}{2} \int_{-1/2}^{1/2} \e^{-(\frac r{\sqrt2\ell}+\pi\mathrm{i} k\sqrt2\ell)^2} \mathrm dr \\
    %& = \sqrt2\ell\e^{-2\pi^2k^2\ell^2} \int_{-(2\sqrt2\ell)^{-1}-\pi\mathrm{i} k\sqrt2\ell}^{(2\sqrt2\ell)^{-1}-\pi\mathrm{i} k\sqrt2\ell} \e^{-t^2} \mathrm dt \\
    %& = \frac{\ell\sqrt{2\pi}}{2} \e^{-2\pi^2k^2\ell^2} \left( \erf\left(\frac{1}{2\sqrt2\ell}-\pi\mathrm{i} k\sqrt2\ell\right) + \erf\left(\frac{1}{2\sqrt2\ell}+\pi\mathrm{i} k\sqrt2\ell\right) \right)\\
    &=\dots=\ell\sqrt{2\pi} \e^{-2\pi^2k^2\ell^2}\, \mathrm{Re}\left[ \erf\left(\frac{1}{2\sqrt2\ell}+\pi\mathrm{i} k\sqrt2\ell\right) \right],
\end{align*}
where $\erf(z) \coloneqq \frac{2}{\sqrt\pi}\int_0^z \e^{-t^2}\mathrm dt$ is the complex-valued error function.
The final result is obtained by making use of simple integration techniques as well as symmetry properties of the error function, that is $\erf(-z)=-\erf(z)$ and $\erf(\overline z)=\overline{\erf(z)}$.
The complex-valued error function is rather difficult to evaluate numerically and is also hard to handle analytically. For the calculation of the approximating Fourier coefficients, we prefer to use the FFT in practice, as described above. We also do not use this analytical form of the Fourier coefficients in our error estimation later on. It is only given here for the sake of completeness.

Note that
\begin{align*}
    \kappa_\text{gauss}'(r) = -\frac{r}{\ell^2} \kappa_\text{gauss}(r), \quad
    \kappa_\text{gauss}''(r) = \left(\frac{r^2}{\ell^4}-\frac{1}{\ell^2}\right) \kappa_\text{gauss}(r),
\end{align*}
that is for the kernel $\kappa_\text{der}(r) \coloneqq r^2/(2\ell^2)\e^{-r^2/2\ell^2}$ we obtain 
\begin{equation*}
    c_k(\tilde\kappa_\text{der}) = \frac{1}{2} \left( c_k(\tilde\kappa_\text{gauss}) + \ell^2 c_k(\tilde\kappa''_\text{gauss}) \right)
    = \frac12 \left(1 - 4\pi^2k^2\ell^2 \right) c_k(\tilde\kappa_\text{gauss}),
\end{equation*}
where we apply the well-known differentiation properties for Fourier series.

So far we just considered the univariate case, where we approximate a certain kernel in terms of $m$ Fourier coefficients. 
Considering uniform grids in higher dimensions, the number of coefficients $m^d$ on the grid grows exponentially fast.
Thus, the computational efficiency of the NFFT approach pays off most for rather small input-dimensions, say $d<4$. As the presented method is designed for large-scale data with many features, a strategy on how to arrange small groups of feature combinations and to detect the most relevant ones is required. By this, several fast NFFT multiplications each relying on a small number of features can be combined via the additive kernel setting as introduced above.

% Subsection: Scaling the Data
\subsection{Scaling the Data} \label{Sec:Scaling}

As described above, we make use of periodic functions and Fourier approximations in order to compute the matrix-vector products efficiently.
Since we work in a periodic setting, that is on a finite interval and not on $\R$, we have to ensure that the data points are scaled into a finite interval.

In order to apply the fast summation approach, as explained above, we have to scale the data such that $\|\x_i-\x_j\|\leq \frac12$ for all pairs $i,j$, which is fulfilled if all the data are scaled such that $\|\x_j\|\leq \frac 14$. Therefore, the $d$-dimensional data points are scaled such that $\x_j\in[-1/4,1/4]^d$ first.
If we denote by $d_\text{max} = \max_{s} d_s$ the maximal number of features incorporated in the sub-kernels, then the maximum norm of a data point, restricted to a set of $d_\text{max}$ features, is given by
$$\Delta_\text{max}=\frac{\sqrt{d_\text{max}}}{4}.$$
Thus, we define the scaled nodes via
$$\tilde\x_j \coloneqq \frac{1}{4}\cdot \frac{\x_j}{\Delta_\text{max}}=\frac{\x_j}{\sqrt{d_\text{max}}}.$$
If a length-scale parameter $\ell$ has already been chosen to be applied to the nodes $\x_j\in[-1/4,1/4]^d$, we scale it with the same scaling factor, that is $\tilde\ell:=\ell/\sqrt{d_\text{max}}$, so that
\begin{equation*}
    \frac{\|\x_i-\x_j\|_2^2}{2\ell^2} = \frac{\|\tilde\x_i-\tilde\x_j\|_2^2}{2\tilde\ell^2}.
\end{equation*}

The advantage of prescaling the data is that the scaling parameter is computed based on the scaled data and is scaled equally in all dimensions. Without this prescaling, the scaling is different for every fastadj object being constructed for the particular windows each. This can turn out to be problematic when performing global sensitivity analysis for instance, see Section~\ref{Sec: Global sensitivity analysis}. Note that the $\ell$ values displayed for the empirical results are the initially chosen length-scale parameters for the data already scaled to $[-1/4,1/4]^d$. The length-scales are then scaled based on the corresponding scaling factor before running the model.

We provide the GitHub repository \texttt{prescaledFastAdj}\footnote{\url{https://github.com/wagnertheresa/prescaledFastAdj}} in which the prescaling is incorporated as described above.

% Subsection: Implementing the NFFT Approach
\subsection{Implementing the NFFT Approach}

Above, we explain the theory behind the fast NFFT-based approximation technique for matrix-vector multiplications with a kernel $K$ and its derivatives. In our setting, $K$ is defined by the additive Gaussian kernel~\eqref{eq:wind_gauss_anova_kernel}
\begin{align}
\label{eq:add_gauss_kernel}
    \kappa ( \bm{x}_i, \bm{x}_j ) = \sigma_f^2 \sum_{s=1}^P \underbrace{\exp \left( - \frac{\| \bm{x}_i^{\mathcal{W}_s} - \bm{x}_j^{\mathcal{W}_s} \|_2^2}{2 \ell^2} \right)}_{\kappa_s}.
\end{align}
As introduced in~\citet{nestlerlearning} this is implemented as a black box approach, where only the data points restricted to the windows $\mathcal{W}_s$, the kernel parameter $\ell$ and a vector $v$, the kernel shall be multiplied with, are required as inputs and the corresponding approximation of $K_s v$ is returned. The underlying kernel is defined as
\begin{align} \label{implement_gaussian_kernel}
    \kappa_s^\text{gauss} (\bm{x}_i^{\mathcal{W}_s}, \bm{x}_j^{\mathcal{W}_s}) = \exp \left( - \frac{\| \bm{x}_i^{\mathcal{W}_s} - \bm{x}_j^{\mathcal{W}_s} \|_2^2}{c^2} \right)
\end{align}
in the implementation, that is with $c = \sqrt{2} \ell$, $\kappa_s^\text{gauss} = \kappa_s$. Summing over several of such approximations each relying on another window $\mathcal{W}_s$ and multiplying this sum by the signal variance parameter $\sigma_f^2$ we obtain $Kv = \sigma_f^2 \sum_{s=1}^P K_s^{\text{gauss}} v$.

% Begin figure RMSE surface
\begin{figure}[ht]
    \centering
    \includegraphics[width=\textwidth,trim=0cm 2.5cm 0cm 2.5cm,clip=true]{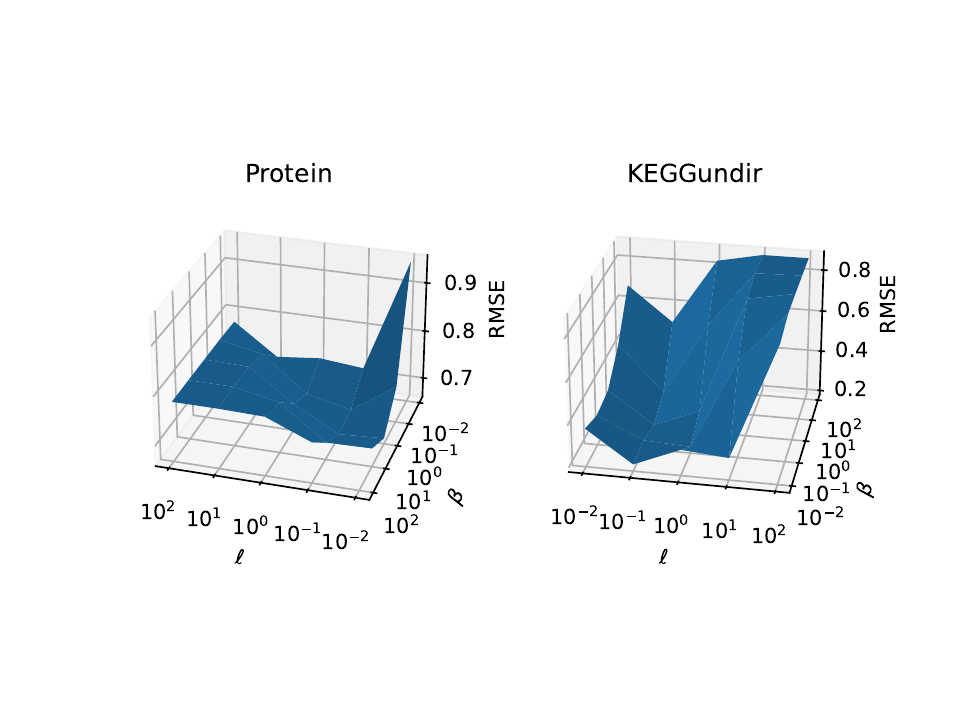}
    \caption{RMSE surface for additive kernel ridge regression and different length-scale and regularization parameters $\ell$ and $\beta$, where $N=1000$ and the windows are determined consecutively via MIS ranking.}
    \label{fig:rmse_surface}
\end{figure}
% End figure RMSE surface

As motivated earlier, the choice of hyperparameters affects the performance of the learning algorithm tremendously. Figure~\ref{fig:rmse_surface} visualizes the impact of the parameter choices on the prediction quality with additive kernel ridge regression (KRR) on the Protein and KEGGundir data sets for instance, where the solution to the system 
\begin{align} \label{eq:regression_system}
    (K + \beta I) v = y,
\end{align}
with $\beta \in \mathbb{R}$ the regularization parameter, is sought. The plot shows the root mean square error (RMSE) on a grid of different values for the length-scale parameter $\ell$ and $\beta$ and highlights the significance of hyperparameter optimization. For more information on the data we refer to Section~\ref{Sec:Numerical results}.

For optimizing the objective of a regression model for instance, the kernel vector product $Kv$ must be differentiated with respect to the kernel parameters $\sigma_f$ and $\ell$. Differentiation with respect to the signal variance can easily be realized for the kernel evaluation $\kappa$ in~\eqref{eq:add_gauss_kernel} since
\begin{align*}
    \frac{\partial K_{ij}}{\partial \sigma_f} = 2 \sigma_f \sum_{s=1}^P \exp \left( - \frac{\| \bm{x}_i^{\mathcal{W}_s} - \bm{x}_j^{\mathcal{W}_s} \|_2^2}{2 \ell^2} \right) = \frac{2}{\sigma_f} K_{ij}.
\end{align*}
Therefore, the multiplication of the derivative kernel $\frac{\partial K}{\partial \sigma_f}$ and $v$ can be performed similarly to the product $Kv$ using the same implementation except that the sum of the approximations $K_sv$ is multiplied by $2\sigma_f$ instead of $\sigma_f^2$, that is\\ $\frac{\partial K}{\partial \sigma_f} v = \frac{2}{\sigma_f} \sigma_f^2 \sum_{s=1}^P K_s^{\text{gauss}}v$. Differentiation with respect to the regularization parameter is straightforward.

However, differentiation with respect to the length-scale parameter $\ell$ gives
\begin{align*}
    \frac{\partial K_{ij}}{\partial \ell} = \sigma_f^2 \sum_{s=1}^P \frac{ \| \bm{x}_i^{\mathcal{W}_s} - \bm{x}_j^{\mathcal{W}_s} \|_2^2}{\ell^3} \exp \left( - \frac{\| \bm{x}_i^{\mathcal{W}_s} - \bm{x}_j^{\mathcal{W}_s} \|_2^2}{2 \ell^2} \right) = \sigma_f^2 \sum_{s=1}^P \underbrace{\frac{C_{s_{ij}}}{\ell^3} \circ K_{s_{ij}}}_{K_{s_{ij}}^{\ell}},
\end{align*}
with $C_{s_{ij}} = \| \bm{x}_i^{\mathcal{W}_s} - \bm{x}_j^{\mathcal{W}_s} \|_2^2$, is more complicated. For the entry-wise multiplication in the Hadamard products $C_s \circ K_s$ many nice properties as the associative law do not hold. Thus, we cannot employ our technique from approximating $Kv$ with the kernel $\kappa_s^\text{gauss}$ as in~\eqref{implement_gaussian_kernel} directly. Instead, we introduce a derivative kernel
\begin{align} \label{implement_der_gaussian_kernel}
    \kappa_s^\text{der} (\bm{x}_i^{\mathcal{W}_s}, \bm{x}_j^{\mathcal{W}_s}) = \frac{ \| \bm{x}_i^{\mathcal{W}_s} - \bm{x}_j^{\mathcal{W}_s} \|_2^2}{c^2} \exp \left( - \frac{\| \bm{x}_i^{\mathcal{W}_s} - \bm{x}_j^{\mathcal{W}_s} \|_2^2}{c^2} \right),
\end{align}
with $c = \sqrt{2}\ell$. For implementation reasons, the parameter in the denominator of both terms in $\kappa_s^\text{der}$ is chosen equally, that is $K_s^\text{der} = \frac{\ell}{2} K_s^{\ell}$ and $K^{\text{der}} v = \frac{\partial K}{\partial \ell} v = \sigma_f^2 \sum_{s=1}^P \frac{2}{\ell} K_s^{\text{der}}v$.

The corresponding implementations can be found in the \texttt{prescaledFastAdj}\\ repository, in which the NFFT-accelerated kernel and derivative kernel evaluations are implemented. $\kappa_s^{\text{gauss}}$ and $\kappa_s^{\text{der}}$ are referred to as $\text{kernel}=1$ and $\text{kernel}=2$, respectively.
Within the `fastsum' module of the underlying \texttt{NFFT}\footnote{\url{https://github.com/NFFT/nfft}} repository~\citep{NFFTrepo}, $\kappa_s^{\text{gauss}}$ is embedded as the `gaussian' kernel and $\kappa_s^{\text{der}}$ as the `xx\_gaussian' kernel.

% Begin figure Fourier approx error
\begin{figure}
    \centering
    \includegraphics[width=\textwidth]{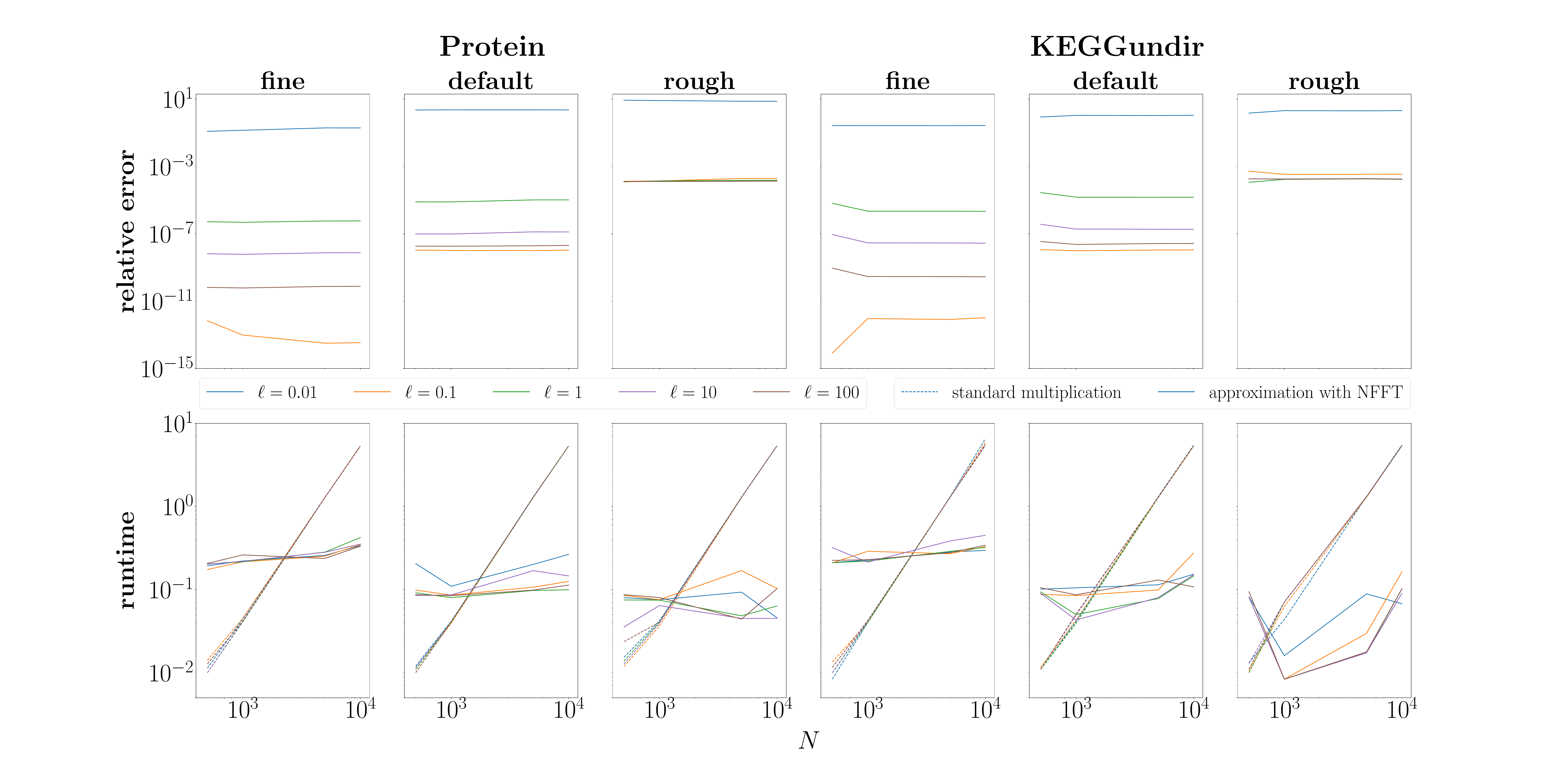}
    \caption{Fourier approximation error for computing $Kv$ with different values of $\ell$, where $v=\bm{1}$, $N_{\text{fg}}=1000$, $d_{\text{max}}=3$ and the windows are determined consecutively via MIS ranking, in comparison with standard multiplication.}
    \label{fig:nfft_approx_error}
\end{figure}
% End figure Fourier approx error

In Figure~\ref{fig:nfft_approx_error} we illustrate the Fourier approximation error for multiplying the kernel $K$ with the $\bm{1}$ vector for subsets of the data sets Protein and KEGGundir, several choices for the length-scale parameter $\ell$ and the three different setup presets for the parameters of the NFFT fastsum method `fine' ($m=64$), `default' ($m=32$) and `rough' ($m=16$). Here, by relative error we denote the relative difference of the Euclidean norms of the exact product $Kv$ and its Fourier approximation. Note that we restrict the size of the considered subsets to $10^4$ at a max since the computations break for bigger matrices in the standard multiplication due to a lack of memory. However, the NFFT-based approximation runs smoothly in such cases. The setups control the number of Fourier coefficients and therefore describe the degree of accuracy in the approximation, where `fine' provides the most precise approximations, `rough' is least precise and `default' is in between. This characteristic is also displayed in the figure. The relative error plots for the `rough' setup are on a higher level than the `default' ones that lie above the `fine' ones for the corresponding parameters $\ell$. Moreover, we compare the runtime for computing $Kv$ via standard multiplication and approximation with the NFFT approach in the second row of the plot. While the NFFT approach has a basic complexity for setting up the fast adjacency object and for computing the Fourier coefficients the runtime does not ascent steeply for larger scales. In contrast, the standard multiplication starts at a very low level for small subset sizes but increases strongly for larger kernels. Whereas the NFFT setups obviously do not impact the runtime plots for the standard multiplication, the runtime of the NFFT approximations differs. The most Fourier coefficients have to be computed in `fine' and the least in `rough'. Therefore, `fine' naturally has a higher runtime than `default' that has a higher runtime than `rough'. While the value of $\ell$ mostly does not seem to have a huge impact on the runtime, the Fourier approximation error evidently highly differs for various values of $\ell$. For very small values of $\ell$ the relative error can become larger than $10$. In contrast, for medium sized values of $\ell$ the relative error ranges between $10^{-3}$ and $10^{-15}$, depending on the setup, and for large values the relative error is between $10^{-4}$ for `rough' and $10^{-10}$ for `fine'.

% Begin figure der Fourier approx error
\begin{figure}
    \centering
    \includegraphics[width=\textwidth]{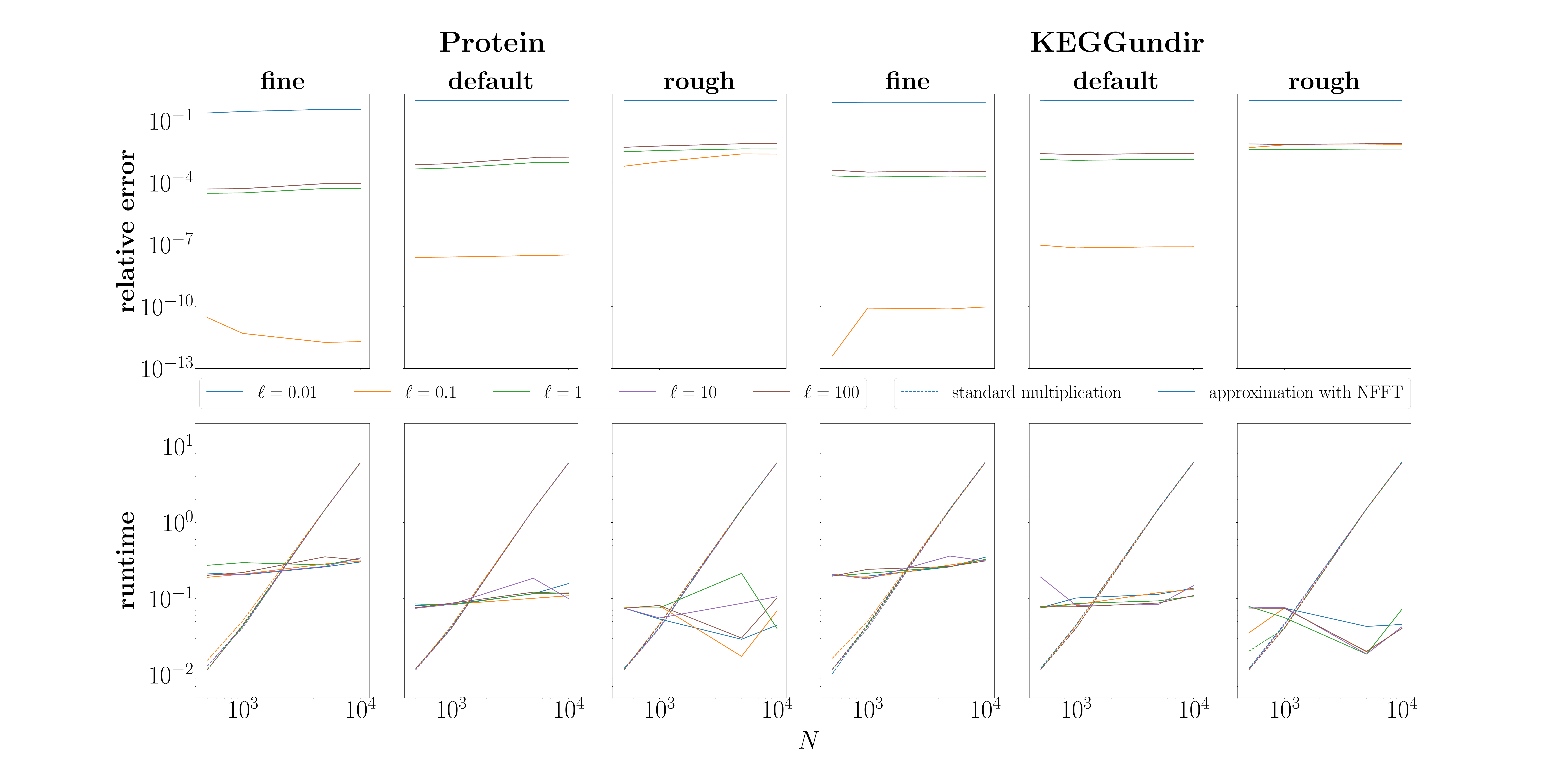}
    \caption{Fourier approximation error for computing $\frac{\partial K}{\partial \ell} v$ with different values of $\ell$, where $v=\bm{1}$, $N_{\text{fg}}=1000$, $d_{\text{max}}=3$ and the windows are determined consecutively via MIS ranking, in comparison with standard multiplication.}
    \label{fig:nfft_der_approx_error}
\end{figure}
% End figure der Fourier approx error

Figure~\ref{fig:nfft_der_approx_error} shows the analogous results for the Fourier approximation error for multiplying the derivative kernel $\frac{\partial K}{\partial \ell}$ with the $\bm{1}$ vector. Overall, the relative errors and runtimes show the same trend as in Figure~\ref{fig:nfft_approx_error}. The main difference is that the relative error is by far the smallest for $\ell=0.1$ for the `fine' and `default' setup presets. In contrast, $\ell=0.01$ clearly yields the largest error and for the length-scales larger or equal $1$ the relative error is at the same level in between. In contrast, the relative errors for length-scales larger or equal $1$ show a greater variation for the distinct values of $\ell$ up to several orders of magnitude in Figure~\ref{fig:nfft_approx_error}.

The relative approximation errors are not satisfactory for all values of $\ell$, of course. The NFFT approach does not approximate the product $Kv$ well when the value of $\ell$ is very small. Note that $\ell$ always appears squared in the denominator of the exponential. With that, very small values $\ell$ lead to kernel matrices $K_s^{\text{gauss}}$ with all entries close to zero except the diagonal being ones. This gives an identity matrix of full rank. The other extreme case is when the values of $\ell$ are very large. Then, all entries in $K_s^{\text{gauss}}$ are close to one what gives a rank $1$ matrix. In the derivative case the approximation error is biggest for very small and very large values of $\ell$. In both cases, all entries of $K_s^{\text{der}}$ are close to zero, what yields a zero matrix of zero rank.

% Begin figure cg iters & eigs
\begin{figure}
    \centering
    \includegraphics[width=\textwidth]{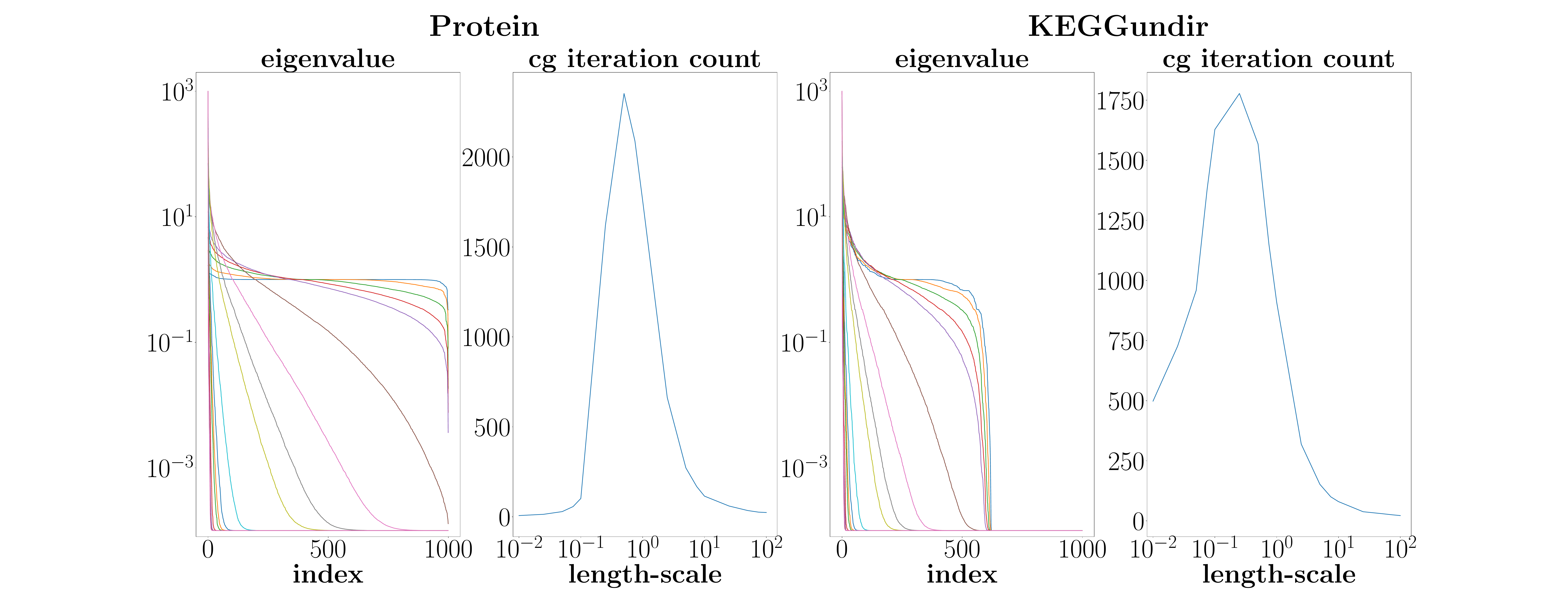}
    \caption{Eigenvalue decay and iteration count of unpreconditioned CG to solve \eqref{eq:regression_system} to reach the relative residual tolerance $10^{-4}$ for fixed $N=1000$ and regularization parameter $\beta=0.0001$ but different length-scales $\ell$, where the windows are determined consecutively via MIS ranking.}
    \label{fig:cg_iters_eigs}
\end{figure}
% End figure cg iters & eigs

When the kernel matrices are of such special structure, multiplying and solving with the kernel matrix is not as challenging. This is emphasized by Figure~\ref{fig:cg_iters_eigs} where the eigenvalue decay is shown alongside the CG iteration count for different values of $\ell$. As before $K$ has beneficial properties for very small and very large values of $\ell$ and as such CG does not require many iterations until convergence. More iterations are needed for very small values $\ell$ than for very large ones for the KEGGundir data set. That effect can be attributed to $K$ being of full rank in contrast to having rank $1$, what is a more difficult system to solve, naturally. More importantly the number of iterations clearly has a peak for moderate values of $\ell$, when $K$ is dense and not of a special structure. In the eigenvalue plots the slowly decaying curves correspond to a large CG iteration count.

Since we are mostly concerned with settings in which the system is not as well-posed and those cases align with values $\ell$ for which the NFFT-approach provides a good Fourier approximation error, the NFFT-based approximation is a competitive acceleration method independently of that effect.

%% Subsection: Fourier Error Estimates
\subsection{Fourier Error Estimates}

After illustrating the empirical Fourier approximation error in different settings above, we derive analytical error estimates next. For that, we assume that the kernels are considered on $[-1/2,1/2]^3$ and periodized by simple periodic continuation, as explained above. Moreover, we assume that the length-scale parameter $\ell$ is already scaled by the scaling factor for the corresponding data points. Since kernels in higher dimensions can be derived via tensor products, the heart of the presented proofs consists of estimating the analytical Fourier coefficients and their sums for the univariate case.
As we exclusively work with dimensions smaller or equal three in the presented additive kernel design, we derive the resulting error estimates for $d=3$.
The case $d=2$ can be treated similarly, whereas the estimates for $d=1$ immediately follow from the computations, see Remarks \ref{rem:gaussian1d} and \ref{rem:deriv1d}.
The error estimates for the Gaussian kernel in two dimensions have already been presented in a slightly different form in \citet{potts2003fast}. Our proofs follow a similar baseline. To the best of our knowledge the Fourier approximation error for the derivative Gaussian kernel, see Theorem~\ref{theorem: multi_deriv}, is estimated for the first time in this paper.

As introduced above in \eqref{eq: fast sum} and \eqref{eq: approx fast sum}, kernel vector products can be represented in a summation form as denoted by $h$ and approximated by a truncated Fourier series denoted by $\tilde{h}$, that is
\begin{equation*}
\begin{aligned}
    h ( \bm{x}_i' ) &\coloneqq \sum_{j=1}^N v_j \kappa ( \bm{x}_i', \bm{x}_j ) \quad \forall i = 1, \dots, N', \\
    h ( \bm{x}_i' )
    \approx  \tilde{h} ( \bm{x}_i' ) &= \sum_{j=1}^{N} v_j \sum_{\bm{k} \in \mathcal{I}_{\bm{M}}} \hat{c}_{\bm{k}} \e^{2 \pi \mathrm{i} \bm{k}^ \intercal ( {\bm{x}}_i' - {\bm{x}}_j )}
    = \sum_{\bm{k} \in \mathcal{I}_{\bm{M}}} \hat{c}_{\bm{k}} \left(\sum_{j=1}^{N} v_j \e^{-2 \pi \mathrm{i} \bm{k}^ \intercal {\bm{x}}_j} \right) \e^{2 \pi \mathrm{i} \bm{k}^ \intercal {\bm{x}}_i'} ,
\end{aligned}
\end{equation*}
for $\x_i',\x_j\in[-1/2,1/2]^3$ already scaled.

Then, by the Hölder inequality the Fourier approximation error is determined by
\begin{equation}
\begin{aligned} \label{eq: NFFT approx error}
    \left| h ( \bm{x}_i') - \tilde{h} ( \bm{x}_i' ) \right| &= \left| \sum_{j=1}^N \alpha_j  \kappa_\text{ERR} ( \bm{x}_i', \bm{x}_j ) \right|
    \le \| \bm{\alpha} \|_1 \| \kappa_\text{ERR} \|_{\infty},
\end{aligned}
\end{equation}
for $i = 1, \dots, N'$ and $\kappa_\text{ERR}$ the difference between the kernel representation by a Fourier series and its approximation by a truncated one, where
\begin{align*}
    \| \bm{\alpha} \|_1
    &\coloneqq \sum_{j=1}^N | \alpha_j |, \\
    \| \kappa_\text{ERR} \|_{\infty}
    &\coloneqq \max_{\x,\x'\in[-1/4,1/4]^d} | \kappa_\text{ERR} (\bm{x}', \bm{x}) |
    = \max_{\br\in[-1/2,1/2]^d} | \kappa_\text{ERR} (\br) |.
\end{align*}
We present theoretical bounds on the achievable approximation error $\|\kappa_\text{ERR}\|_\infty$ in the following, where we set
\begin{equation*}
    \IM := \{\bk\in\Z^d: -\tfrac m2\leq k_j<\tfrac m2 \,\,\forall j=1,\dots, d\}.
\end{equation*}
To this end, we consider the periodized Gaussian and derivative Gaussian kernels in $d=3$ dimensions.
For $\br:=\x'-\x\in[-1/2,1/2]^3$, we obtain
\begin{align*}
    \kappa(\br)= \sum_{\bk\in\Z^3} c_{\bk}(\kappa) \e^{2\pi\mathrm{i}\bk^\intercal\br}
    &\approx \sum_{\bk\in\IM} c_{\bk}(\kappa) \e^{2\pi\mathrm{i}\bk^\intercal\br}
    %=: \kappa_{\text{trunc}}(\br)
    \approx \sum_{\bk\in\IM} \hat c_{\bk} \e^{2\pi\mathrm{i}\bk^\intercal\br} =: \kappa_{\text{F}}(\br),
    %\approx \kappa_{\text{NFFT},m}(\br),
\end{align*}
where $c_{\bk}(\kappa)$ are the analytical Fourier coefficients and $\hat c_{\bk}$ the discrete Fourier coefficients obtained from $m^3$ equidistant samples.
The Fourier approximation error that is studied in this subsection is now defined as
\begin{align*}
    \kappa_\text{ERR}(\br)
    \coloneqq\;&
    \kappa(\br)-\kappa_{\text{F}}(\br).
\end{align*}
Note that the sums $h(\x_i')$ are ultimately not evaluated directly, but approximated by the NFFT algorithms, that is
\begin{equation*}
    \tilde h(\x_i')
    \approx
    \underbrace{
    \sum_{\bm{k} \in \mathcal{I}_m} \hat{c}_{\bm{k}}
    \underbrace{
    \left(\sum_{j=1}^{N} v_j \e^{-2 \pi \mathrm{i} \bm{k}^\intercal {\bm{x}}_j} \right)
    }_\text{approx.\ via adjoint NFFT}
    \e^{2 \pi \mathrm{i} \bm{k}^\intercal {\bm{x}}_i'}
    }_\text{approx.\ via NFFT},
\end{equation*}
what introduces further approximation errors.
However, these approximation errors are neglected at this point and the pure Fourier approximation error is considered in the following.
The NFFT algorithms depend on several parameters controlling the accuracy. Choosing these parameters appropriately, the NFFT approximation errors can be made negligibly small.

\begin{lemma} \label{lemma: aliasing error}
(Aliasing error) Let $f \in L_2 (\mathbb{T}^3)$ be a function with absolutely convergent Fourier series
\begin{align*}
    f( \bm{x} ) = \sum_{\bm{k} \in \mathbb{Z}^3} c_{\bm{k}} (f) \e^{2 \pi \mathrm{i} \bk^\intercal \bm{x}}, \quad c_{\bm{k}} (f) = \int_{\mathbb{T}^3} f(\bm{x}) \e^{-2 \pi \mathrm{i} \bk^\intercal \bm{x}} \, \mathrm{d} \bm{x}
\end{align*}
and let an approximation of $f$ be given by (replacing the analytic Fourier coefficients by the discrete Fourier coefficients using $m$ equidistant samples on the grid $\Im$)
\begin{align*}
    f_\text{F} (\bm{x}) \coloneqq \sum_{\bm{k} \in \mathcal{I}_{m}} \hat{c}_{\bm{k}} \e^{2 \pi \mathrm{i} \bm{k} \bm{x}},
    \quad
    \hat{c}_{\bm{k}} = \frac{1}{m^3} \sum_{\bm{j} \in \mathcal{I}_{m}} f \left( \frac{\bm{j}}{m} \right) \e^{-2\pi \mathrm{i} \bm{j} \bm{k}/m}.
\end{align*}
Then 
\begin{align}\label{eq:aliasing_formula}
    \hat{c}_{\bm{k}} = c_{\bm{k}} (f) + \sum_{\bm{r} \in \mathbb{Z}^3 \setminus \{\bm{0}\}} c_{\bm{k}+m \bm{r}} (f) 
\end{align}
and the approximation error can be estimated for all $\bm{x} \in \mathbb{T}^3$ by
\begin{align*}
%\label{eq:aliasing_error}
    \left| f(\bm{x}) - f_\text{F}(\bm{x}) \right| \le 2 \sum_{\bm{r} \in \mathbb{Z}^3 \setminus \{\bm{0}\}} \sum_{\bm{k} \in \mathcal{I}_{m}} \left| c_{\bm{k}+m \bm{r}} (f) \right|
    = 2\sum_{\bm{k} \in \Z^3\setminus\mathcal{I}_{m}} \left| c_{\bm{k}+m \bm{r}} (f) \right|.
\end{align*}
\end{lemma}

\begin{proof}
    For the derivation of the well-known aliasing formula \eqref{eq:aliasing_formula}, which states the relationship between the analytic and the discrete Fourier coefficients, we refer to the standard literature on Fourier analysis, see for instance \citet{PlPoStTa18} and references therein.

    The stated estimate between $f$ and $f_F$ is then a simple consequence of the triangle inequality, as sketched in the following.
    We rewrite the function $f$ as
    \begin{equation*}
        f(\x) = \sum_{\bk\in\Im} c_{\bk}(f)\e^{2\pi\mathrm{i}\bk^\intercal\x} +\sum_{\bk\in\Im}\sum_{\bm{r}\in\Z^3\setminus\{\bm{0}\}} c_{\bk+m\bm r}(f) \e^{2\pi\mathrm{i}(\bk+m\bm r)^\intercal\x}
    \end{equation*}
    and its approximation $f_F$ as
    \begin{equation*}
        f_F(\x) = \sum_{\bk\in\Im} c_{\bk}(f)\e^{2\pi\mathrm{i}\bk^\intercal\x} + \sum_{\bk\in\Im}\sum_{\bm{r}\in\Z^3\setminus\{\bm{0}\}} c_{\bk+m\bm r}(f) \e^{2\pi\mathrm{i}\bk^\intercal\x},
    \end{equation*}
    and conclude
    \begin{equation*}
        f(\x)-f_F(\x) = \sum_{\bk\in\Im}\sum_{\bm{r}\in\Z^3\setminus\{\bm{0}\}}  \left(\e^{2\pi\mathrm{i}(\bk+m\bm r)^\intercal\x} -\e^{2\pi\mathrm{i}\bk^\intercal\x} \right) c_{\bk+m\bm r}(f).
    \end{equation*}
    Consequently, applying the triangle inequality gives 
    \begin{equation*}
        |f(\x)-f_F(\x)| \leq \sum_{\bk\in\Im}\sum_{\bm{r}\in\Z^3\setminus\{\bm{0}\}} 2\cdot |c_{\bk+m\bm r}(f)|.
    \end{equation*}
\end{proof}

% Subsubsection: Gaussian Kernel
\subsubsection{Gaussian Kernel}
We review the theoretical results presented in \citet{postni04}, where the authors consider the Gaussian kernel in two variables and derive an upper bound for $\|\kappa_{\text{ERR}}\|_\infty$.
This result can be extended to higher dimensions, where formulas follow the same rules but become somewhat more extensive.
In the following theorem we improve the error estimates of~\citet{postni04} and restrict our considerations to the case $d=3$.

\begin{theorem} \label{theorem: multi_gauss}
    Let the kernel matrix be defined by the trivariate Gaussian $\kappa_s^{\text{\emph{gauss}}}$ in \eqref{implement_gaussian_kernel}. Then, for $\eta \coloneqq \frac{\ell \pi m}{ \sqrt{2}} \ge 1$ the following estimate holds true
    \begin{equation*}
    \|\kappa_{\text{\emph{ERR}}}\|_\infty
    \leq 15 \gamma(\eta,\ell) \left(\gamma(\eta,\ell) +\frac52 \right) +
    102 A(\eta,\ell),
    \end{equation*}
    where
    \begin{align*}
        \gamma(\eta,\ell)
        &:= \ell\sqrt{2\pi}\e^{-\eta^2} +
        \begin{cases}
        \ds \frac{\e^{-1/8\ell^2}}{\eta^2}
        %\leq \frac{2\ell\e^{-1/2}}{\eta^2}
        &: \ell<\frac 12 \\
        \ds \frac{\ell\e^{-1/2}}{\eta^2} &: \ell\geq\frac12
        \end{cases}\\
        \intertext{and}
        A(\eta,\ell)
        &:= \frac{1}{2\eta\sqrt\pi}\e^{-\eta^2} +
        \begin{cases}
        \ds \frac{\e^{-1/8\ell^2}}{\sqrt2 \ell\pi\eta}
        %\leq \frac{\sqrt2\e^{-1/2}}{\pi\eta}
        &: \ell<\frac 12 \\
        \ds \frac{\sqrt2\e^{-1/2}}{\pi\eta} &: \ell\geq \frac12
        \end{cases}.
        \\
    \end{align*}
\end{theorem}

\begin{proof}
Throughout this proof we make use of the short hand notation $f(x) \coloneqq \e^{-x^2/2\ell^2}$.
Further, we will employ the following simple estimates
\begin{align}
    \sum_{k=1}^{m/2} \frac{1}{k^2} &\le \sum_{k=1}^{\infty} \frac{1}{k^2} = \frac{\pi^2}{6}, \label{eq: se_kernel aux_estimates1} \\
    \sum_{k=1}^{m/2} \e^{-2k^2\pi^2\ell^2} &\le \int_0^{\infty} \e^{-x^2\pi^22\ell^2} \, \mathrm{d}x = \frac{1}{2\ell \sqrt{2\pi}}, \label{eq: se_kernel aux_estimates2} \\
    \sum_{k=m/2+1}^{\infty} \frac{1}{k^2} &\le \int_{m/2}^{\infty} \frac{1}{x^2} \, \mathrm{d}x = \frac{2}{m}, \label{eq: se_kernel aux_estimates3} \\
    \sum_{k=m/2+1}^{\infty} \e^{-2k^2\pi^2\ell^2} &\le \int_{m/2}^{\infty} \e^{-x^2\pi^22\ell^2} \, \mathrm{d}x \le \frac{1}{2\ell^2 m \pi^2} \e^{-\pi^2m^2\ell^2/2}, \label{eq: se_kernel aux_estimates4}
\end{align}
where the last line follows from
\begin{align*}
    \int_a^{\infty} \e^{-cx^2} \, \mathrm{d}x \le \int_0^{\infty} \e^{-c(x+a)^2} \, \mathrm{d}x \le \e^{-ca^2} \int_0^{\infty} \e^{-2acx} \, \mathrm{d}x = \frac{\e^{-ca^2}}{2ac}
\end{align*}
and the estimates \eqref{eq: se_kernel aux_estimates2}--\eqref{eq: se_kernel aux_estimates4} are simply obtained by estimating the sum from above by an integral over a monotonically decreasing function.

The Fourier transform of the univariate Gaussian is defined as
\begin{align}\label{eq:FT_gauss}
    \hat f(k):=
    \int_{-\infty}^{\infty} f(x) \e^{-2\pi \mathrm{i} kx} \, \mathrm{d}x = \ell \sqrt{2\pi} \e^{-2k^2\pi^2\ell^2}.
\end{align}
By applying integration by parts twice, we obtain for the Fourier coefficients and $k \neq 0$
\begin{equation*}
\begin{aligned}
    c_k (f) &\coloneqq \int_{-1/2}^{1/2} f(x) \e^{-2\pi \mathrm{i} k x} \, \mathrm{d}x \\
    &= (-1)^{k+1} \frac{1}{4\ell^2 \pi^2 k^2} \e^{-1/8\ell^2} - \frac{1}{4 \pi^2 k^2} \int_{-1/2}^{1/2} f''(x) \e^{-2 \pi \mathrm{i} k x} \, \mathrm{d}x \\
    &= (-1)^{k+1} \frac{1}{4\ell^2 \pi^2 k^2} \e^{-1/8\ell^2} - \frac{1}{4\pi^2 k^2} \int_{-\infty}^{\infty} f''(x) \e^{-2 \pi \mathrm{i} k x} \, \mathrm{d}x \\
    & \quad + \frac{1}{2\pi^2 k^2} \int_{1/2}^{\infty} f''(x) \cos (2\pi kx) \, \mathrm{d}x,
\end{aligned}
\end{equation*}
where $f''(x)=\ell^{-2}\e^{-x^2/2\ell^2}(\ell^{-2}x^2-1)$.
The second last integral is simply the Fourier transform of $f''$, which is given by $4\pi^2k^2 \hat f(k)$.

In order to obtain an estimate for $|c_k(f)|$ we may simply use the triangle inequality and it remains to estimate
\begin{equation*}
    \left|\int_{1/2}^{\infty} f''(x) \cos (2\pi kx) \, \mathrm{d}x\right|
    \leq \int_{1/2}^{\infty} |f''(x)| \, \mathrm{d}x
\end{equation*}
further.
First, we note that $f''$ changes its sign in $x_0:=\ell$ and, thus, the value of the integral depends on whether $\ell\geq\frac 12$ or $\ell<\frac 12$.

If $\ell<\frac 12$ we compute
\begin{equation*}
    \int_{1/2}^{\infty} |f''(x)| \, \mathrm{d}x = -f'(\tfrac 12) = \frac{\e^{-1/8\ell^2}}{2\ell^2}
\end{equation*}
and for $\ell\geq\frac 12$ it holds
\begin{equation*}
    \int_{1/2}^{\infty} |f''(x)| \, \mathrm{d}x = f'(\tfrac 12) - 2f'(\ell) = \frac{2\e^{-1/2}}{\ell} - \frac{\e^{-1/8\ell^2}}{2\ell^2} \geq 0.
\end{equation*}
Therefore,
\begin{equation*}
\begin{aligned} %\label{proof: modulus Fourier coeff se_kernel}
    \left| c_k (f) \right| &\leq \ell\sqrt{2\pi} \e^{-2\pi^2\ell^2k^2} +
    \begin{cases}
        \displaystyle  \frac{\e^{-1/8\ell^2}}{2\ell^2\pi^2k^2}
        &: \ell<\frac12 \\
        \displaystyle  \frac{\e^{-1/2}}{\pi^2k^2\ell}
        &: \ell\geq\frac12
    \end{cases}.
    %\le \frac{1}{4\ell^2 \pi^2 k^2} e^{-1/8\ell^2} + \ell \sqrt{2\pi} e^{-2\ell^2 \pi^2 k^2}
    %+\frac{1}{4\pi^2k^2\ell^2}
    %\begin{cases}
    %    \e^{-1/8\ell^2} &: \ell<\frac12 \\
    %    4\ell \e^{-1/2}- \e^{-1/8\ell^2} &: \ell\geq \frac12
    %\end{cases}
\end{aligned}
\end{equation*}
With that, we conclude
\begin{align*}
    |c_{m/2}(f)| &\leq \ell\sqrt{2\pi}\e^{-\eta^2} +
    \begin{cases}
        \ds \frac{\e^{-1/8\ell^2}}{\eta^2}
        %\leq \frac{2\ell\e^{-1/2}}{\eta^2}
        &: \ell<\frac 12 \\
        \ds \frac{2\ell\e^{-1/2}}{\eta^2}
        &: \ell\geq\frac12
    \end{cases} =: \gamma(\eta,\ell).
    %&\leq \ell\sqrt{2\pi}\e^{-\eta^2} + \frac{2\ell\e^{-1/2}}{\eta^2} \eqqcolon \gamma(\eta,\ell)
\end{align*}
%where we use that $x^{-1}\e^{-1/8x^2}\leq 2\e^{-1/2} \iff x\e^{-1/2}\geq \e^{-1/8x^2}$ for all $x>0$.
Next, applying \eqref{eq: se_kernel aux_estimates1} and \eqref{eq: se_kernel aux_estimates2} we compute
\begin{align*}
    \sum_{k=1}^{m/2} |c_k(f)|
    &\leq \frac 12 +
    \begin{cases}
        \ds \frac{\e^{-1/8\ell^2}}{12\ell^2} \leq \frac14 &: \ell<\frac12 \\
        \ds \frac{\e^{-1/2}}{6\ell} \leq \frac{\e^{-1/2}}{3}\approx 0.202 &: \ell\geq\frac 12
    \end{cases}
    \quad \leq \frac 34,
\end{align*}
where we make use of $x^{-2}\e^{-1/8x^2}\leq 8\e^{-1}\approx 2.94<3$ for all $x\in\R$.
Further, with \eqref{eq: se_kernel aux_estimates3} and \eqref{eq: se_kernel aux_estimates4} we deduce
\begin{align*}
    \sum_{k=m/2+1}^{\infty} |c_k(f)|
    &\leq \frac{1}{2\eta\sqrt\pi}\e^{-\eta^2} +
    \begin{cases}
        \ds \frac{\e^{-1/8\ell^2}}{\sqrt2 \ell\pi\eta}
        %\leq \frac{\sqrt2\e^{-1/2}}{\pi\eta}
        &: \ell<\frac 12 \\
        \ds \frac{\sqrt2\e^{-1/2}}{\pi\eta} &: \ell\geq \frac12
    \end{cases} =: A(\eta,\ell).
    %\\
    %&\leq \frac{1}{2\eta\sqrt\pi}\e^{-\eta^2} + \frac{\sqrt2\e^{-1/2}}{\pi\eta} \eqqcolon A(\eta),
\end{align*}
Note that for $\eta\geq 1$ we obtain $A(\eta,\ell)\leq (2\e\sqrt\pi)^{-1}+\sqrt2\pi^{-1}\e^{-1/2}\approx 0.377$, where we make use of the fact that the function $x^{-1}\e^{-1/8x^2}$ is monotonically increasing on $(0,1/2)$.
 
Based on Lemma~\ref{lemma: aliasing error} and by exploiting the underlying symmetry we have
\begin{equation*}
\begin{aligned}
    \|\kappa_{\text{ERR}} \|_{\infty} \leq\,& 2 \cdot \sum_{\bk\in\Z^3\setminus\Im} |c_{k_1}(f)| |c_{k_2}(f)| |c_{k_3}(f)|\\
    =\,&
    2\cdot 3 |c_{m/2}(f)| \Bigl(\sum_{k=-m/2}^{m/2} |c_k(f)|\Bigr)^2 
    + 2\cdot 3 |c_{m/2}(f)|^2 \sum_{k=-m/2}^{m/2} |c_k(f)|\\
    \,&+ 2\sum_{\|\bk\|_\infty \geq m/2+1} |c_{k_1}(f)| |c_{k_2}(f)| |c_{k_3}(f)|
    \eqqcolon 6S_1+6S_2+2S_3,
\end{aligned}
\end{equation*}
where we exploit the tensor product structure $c_{\bm k}(\e^{-\|\x\|^2/2\ell^2})=c_{k_1}(f)c_{k_2}(f)c_{k_3}(f)$.
Based on the above derived estimates and by using $|c_0(f)|<1$ we obtain
\begin{align*}
    S_1 &\leq \gamma(\eta,\ell) \left(1+2\cdot \tfrac 34\right)^2 = \frac{25}4 \gamma(\eta,\ell), \\
    S_2 &\leq \gamma(\eta,\ell)^2 \left(1+2\cdot \tfrac 34\right) = \frac 52 \gamma(\eta,\ell)^2
\end{align*}
and
\begin{align*}
    S_3 \leq\,&
    8 \left(\sum_{k=m/2+1}^\infty |c_k(f)|\right)^3 + 4\cdot 3 \left(\sum_{k=-m/2}^{m/2} |c_k(f)|\right) \left(\sum_{k=m/2+1}^{\infty} |c_k(f)|\right)^2\\
    \,& +2\cdot 3 \left(\sum_{k=-m/2}^{m/2} |c_k(f)|\right)^2 \left(\sum_{k=m/2+1}^{\infty} |c_k(f)|\right) \\
    \leq\,& 8 A(\eta,\ell)^3 + 12 \cdot \frac{5}{2} A(\eta,\ell)^2 + 6 \cdot \frac{25}{4}A(\eta,\ell).
\end{align*}
%As stated above, we have $A(\eta)<1$ for $\eta\geq 1$ and, consequently $A(\eta)<\ell\sqrt{2\pi}+\frac 32$.
Now, we summarize
\begin{align*}
    \|\kappa_{\text{ERR}}\|_\infty
    &\leq \frac{75}{2}\gamma(\eta,\ell) + 15\gamma(\eta,\ell)^2 + 
    16 A(\eta,\ell)^3 + 60 A(\eta,\ell)^2 + 75A(\eta,\ell).
\end{align*}
Since $A(\eta,\ell)<\frac{2}{5}$, we have $A(\eta,\ell)^2<\frac{2}{5}A(\eta,\ell)$ and $A(\eta,\ell)^3<\frac{4}{25}A(\eta,\ell)$.
With that, we obtain a somewhat more simple estimate of the form
\begin{equation*}
    \|\kappa_{\text{ERR}}\|_\infty < 15 \gamma(\eta,\ell) \left(\gamma(\eta,\ell) +\frac52 \right) +
    \underbrace{\frac{2539}{25}}_{<102} A(\eta,\ell).
\end{equation*}
\end{proof}

We have now established a rigorous error bound for one sub-kernel $\kappa_s^{\text{gauss}}$. The result for the additive kernel $\kappa$ follows straightforwardly by applying the bound to each kernel individually.

\begin{remark}\label{rem:gaussian1d}
    Considering the one-dimensional case with $\kappa^{(1d)}(r):=\e^{-r^2/(2\ell^2)}$ we obtain
    \begin{align*}
        \|\kappa_{\text{\emph{ERR}}}^{(1d)}\|_\infty
        &\leq 2|c_{m/2}(f)| + 4\sum_{k=m/2+1}^\infty |c_k(f)| \\
        &\leq 2\gamma(\eta,\ell) + 4 A(\eta,\ell),
    \end{align*}
    with $\gamma(\eta,\ell)$ and $A(\eta,\ell)$ as stated above.

    For very small values of the kernel parameter $\ell$, see left plot in Figure~\ref{fig:per_gauss_1d}, the periodized kernel can be considered to be smooth and the exponential decay dominates the error bound. However, the smaller the kernel parameter $\ell$, the slower the Fourier coefficients decrease to zero.
    Meaningful values for $m$ are limited in order to achieve $\eta\geq1.$

    In the case of moderate values of $\ell$, see second plot in Figure~\ref{fig:per_gauss_1d}, the periodized kernel has a sharp kink and the terms $\sim\eta^{-2}$ and $\sim\eta^{-1}$ have a greater influence on the estimate.
    For large values of $\ell$, see right plot in Figure~\ref{fig:per_gauss_1d}, a constant kernel with zero error is approached, since essentially only $c_0(f)\neq0$.

    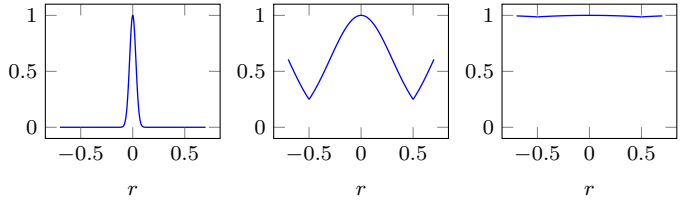
\begin{figure}[ht]
        \centering
        % l small, value 0.03
        \begin{tikzpicture}\footnotesize
        \begin{axis}
        [width=0.3\linewidth,ymin=-0.1,ymax=1.1,xlabel={$r$}]
        \addplot[color=blue,solid,line width=0.5pt,domain=-0.5:0.5,samples=300]{exp(-(x^2/(2*0.03^2))))};
        \addplot[color=blue,solid,line width=0.5pt,domain=-0.7:-0.5,samples=60]{exp(-((x+1)^2/(2*0.03^2))))};
        \addplot[color=blue,solid,line width=0.5pt,domain=0.5:0.7,samples=60]{exp(-((x-1)^2/(2*0.03^2))))};
        \end{axis}
        \end{tikzpicture}
        % l intermediate, value 0.3
        \begin{tikzpicture}\footnotesize
        \begin{axis}
        [width=0.3\linewidth,ymin=-0.1,ymax=1.1,xlabel={$r$}]
        \addplot[color=blue,solid,line width=0.5pt,domain=-0.5:0.5,samples=300]{exp(-(x^2/(2*0.3^2))))};
        \addplot[color=blue,solid,line width=0.5pt,domain=-0.7:-0.5,samples=60]{exp(-((x+1)^2/(2*0.3^2))))};
        \addplot[color=blue,solid,line width=0.5pt,domain=0.5:0.7,samples=60]{exp(-((x-1)^2/(2*0.3^2))))};
        \end{axis}
        \end{tikzpicture}
        % l large, value 3.0
        \begin{tikzpicture}\footnotesize
        \begin{axis}
        [width=0.3\linewidth,ymin=-0.1,ymax=1.1,xlabel={$r$}]
        \addplot[color=blue,solid,line width=0.5pt,domain=-0.5:0.5,samples=300]{exp(-(x^2/(2*3^2))))};
        \addplot[color=blue,solid,line width=0.5pt,domain=-0.7:-0.5,samples=60]{exp(-((x+1)^2/(2*3^2))))};
        \addplot[color=blue,solid,line width=0.5pt,domain=0.5:0.7,samples=60]{exp(-((x-1)^2/(2*3^2))))};
        \end{axis}
        \end{tikzpicture}
    \caption{Periodized Gaussian kernels in 1D, with parameters $\ell=0.03$, $\ell=0.3$ and $\ell=3$ (from left to right).
    \label{fig:per_gauss_1d}}
    \end{figure}
\end{remark}

% Subsubsection: Derivative Gaussian Kernel
\subsubsection{Derivative Gaussian Kernel}

For the derivative Gaussian kernel we again consider the case $d=3$, where
\begin{equation*}
    \frac{\|\x\|_2^2}{2\ell^2} \e^{-\|\x\|_2^2/2\ell^2} =
    \frac{x_1^2}{2\ell^2}\e^{-x_1^2/2\ell^2}\e^{-x_2^2/2\ell^2}\e^{-x_3^2/2\ell^2}
    +\dots+
    \e^{-x_1^2/2\ell^2}\e^{-x_2^2/2\ell^2} \frac{x_3^2}{2\ell^2}\e^{-x_3^2/2\ell^2}
\end{equation*}
and we obtain
\begin{equation*}
    c_\bk\left(\|\cdot\|_2^2 \e^{-\|\cdot\|_2^2/2\ell^2}\right) =
    c_{k_1}(g) c_{k_2}(f) c_{k_3}(f)
    + c_{k_1}(f) c_{k_2}(g) c_{k_3}(f)
    + c_{k_1}(f) c_{k_2}(f) c_{k_3}(g)
\end{equation*}
for the Fourier coefficients, where we define $f(x)\coloneqq\e^{-x^2/2\ell^2}$ and $g(x)\coloneqq \frac{x^2}{2\ell^2} \e^{-x^2/2\ell^2}$.

\begin{theorem}\label{theorem: multi_deriv}
    Let the kernel matrix be defined by the trivariate derivative Gaussian $\kappa_s^{der}$ in~\eqref{implement_der_gaussian_kernel}.
    Then, for $\eta:=\frac{\ell\pi m}{\sqrt2}\geq 1$ the following estimate holds true
    \begin{align*}
        \left\|\kappa_{\text{\emph{ERR}}}\right\|_\infty
        <\;&
        \left(\frac52 \xi(\eta,\ell)+15\gamma(\eta,\ell)\right) \left(15+12\gamma(\eta,\ell)\right) \\
        &+ 75 S(\eta,\ell) +  6 A(\eta,\ell)\left(\frac{116}{5}S(\eta,\ell)+87\right),
    \end{align*}
    where $A(\eta,\ell)$ and $\gamma(\eta,\ell)$ are stated in Theorem~\ref{theorem: multi_gauss} and
    \begin{align*}
        \xi(\eta,\ell)
        &\coloneqq \left(\eta^2+\tfrac12\right)\ell\sqrt{2\pi} \e^{-\eta^2} +
        \begin{cases}
            \ds\frac{\e^{-1/8\ell^2}}{8\eta^2\ell^2} &: \ell\leq\frac12\sqrt{\frac{2}{5+\sqrt{17}}} \\
            \ds\frac{1}{\eta^2}+\frac{3\ell}{2\eta^2} &: \text{else}
        \end{cases},\\
        S(\eta,\ell)
        &:= \frac{\erfc(\eta)}{4} + \frac{\eta}{2\sqrt\pi}\e^{-\eta^2} + \frac{1}{4\sqrt\pi \eta} \e^{-\eta^2} + 
        \begin{cases}
            \ds\frac{\e^{-1/8\ell^2}}{8\sqrt2\pi\ell^3\eta} &: \ell\leq\frac12\sqrt{\frac{2}{5+\sqrt{17}}} \\
            \ds\frac{1}{\sqrt2\pi\ell\eta}+\frac{3}{2\sqrt2\pi\eta} &: \text{else}
        \end{cases}.
    \end{align*}
\end{theorem}
\begin{proof}
    Throughout this proof we make use of the short hand notations $f(x)= \e^{-x^2/2\ell^2}$ and $g(x)= \frac12 x^2\ell^{-2} \e^{-x^2/2\ell^2}$, as already introduced above.
    First, we compute the Fourier transform of the function $g$. We see $g(x)= \frac12\ell^2 f''(x) + \frac12 f(x)$ and conclude
    \begin{equation}\label{eq:ghat}
        \hat g(k)=\int_{-\infty}^\infty g(x)\e^{-2\pi\mathrm{i} kx}\mathrm dx = (-2\pi^2k^2\ell^2+\tfrac12) \hat f(k),
    \end{equation}
    where $\hat f(k)= \ell\sqrt{2\pi} \e^{-2k^2\pi^2\ell^2}$, as stated in \eqref{eq:FT_gauss}.
    In addition to the estimates \eqref{eq: se_kernel aux_estimates1}--\eqref{eq: se_kernel aux_estimates4}, we will make use of the following additional estimates
    \begin{align}
        \sum_{k=m/2+1}^\infty k^2\e^{-2\pi^2k^2\ell^2} &\leq \int_{m/2}^\infty x^2\e^{-2\pi^2x^2\ell^2} \mathrm dx
        =\frac{1}{\sqrt{2\pi}} \frac{\erfc\left(\pi\ell m/\sqrt2\right)}{8\pi^2\ell^3} + \frac{m}{8\pi^2\ell^2} \e^{-\pi^2\ell^2m^2/2} \label{eq:estimate1_der},\\
        \sum_{k=1}^{m/2} k^2\e^{-2\pi^2k^2\ell^2}
        %&\leq \int_{0}^\infty x^2\e^{-2\pi^2x^2\ell^2}\mathrm dx +
        %\begin{cases}
        %    \ds\frac{\e^{-1}}{2\pi^2\ell^2} &:\ds\ell\leq\frac{1}{\sqrt2\pi}\\
        %    \ds\e^{-2\pi^2\ell^2} &: \text{else}
        %\end{cases} \\
        &\leq \frac{1}{\sqrt{2\pi}\cdot 8\pi^2\ell^3} +
        \begin{cases}
            \ds\frac{\e^{-1}}{2\pi^2\ell^2} &:\ds\ell\leq\frac{1}{\sqrt2\pi}\\
            \ds\e^{-2\pi^2\ell^2} &: \text{else}
        \end{cases}
        .
        \label{eq:estimate2_der}
    \end{align}
    These estimates are obtained as follows.
    Note that the function $h(x) \coloneqq x^2\e^{-2\pi^2x^2\ell^2}$ with $h'(x)=2x \e^{-2\pi^2x^2\ell^2} (1-2\pi^2\ell^2x^2)$ is monotonically decreasing for $x\geq (\sqrt2 \pi\ell)^{-1}$.
    Thus, presuming $\frac{\ell\pi m}{\sqrt 2}\geq 1$ we are able to estimate the first sum by the stated integral. The well-known complementary error function is defined by
    \begin{align*}
        \erfc(x)\coloneqq 1-\erf(x)=1-\frac{2}{\sqrt\pi}\int_0^x \e^{-t^2}\mathrm dt
    \end{align*}
    and tends to zero exponentially fast for $x\to\infty$.
    
    The second estimate is obtained as follows. The function $x^2\e^{-2\pi^2x^2\ell^2}$ is monotonically increasing from $0$ to $(\sqrt{2}\pi l)^{-1}$, where we may estimate the sum above by the integral of the shifted function.
    In the area where the function decreases, we obtain an upper estimate by the integral of the function itself.
    However, the position of the maximum depends on the parameter $\ell$. Thus, let $(\sqrt2\pi\ell)^{-1}\geq1 \iff \ell\leq(\sqrt2\pi)^{-1}$ and $k_\ell\in\N_0$ be the largest natural number $\leq (\sqrt{2}\pi l)^{-1}$. Then, we obtain an upper estimate via
    \begin{equation*}
         \int_0^{k_\ell} (x+1)^2\e^{-2\pi^2(x+1)^2\ell^2}\mathrm dx
         + \frac{\e^{-1}}{2\pi^2\ell^2}
         + \int_{k_\ell+1}^\infty x^2\e^{-2\pi^2x^2\ell^2}\mathrm dx
        <\frac{\e^{-1}}{2\pi^2\ell^2} + \int_0^\infty x^2\e^{-2\pi^2x^2\ell^2}\mathrm dx,
    \end{equation*}
    which consists of two integrals estimating the sums up to $k_\ell$ and from $k_\ell+2$, respectively, plus the area of the rectangle of with $1$ and height $(2\pi^2\ell^2\e)^{-1}$ being the maximum value of the integrated function.
    
    For large $\ell>(\sqrt2\pi)^{-1}$ we obtain the estimate
    \begin{equation*}
        \e^{-2\pi^2\ell^2}+\int_1^\infty x^2\e^{-2\pi^2x^2\ell^2} \mathrm dx < \e^{-2\pi^2\ell^2}+\int_0^\infty x^2\e^{-2\pi^2x^2\ell^2} \mathrm dx,
    \end{equation*}
    what is simply the area of the rectangle of width $1$ and height $\e^{-2\pi^2\ell^2}$, what is the function value at $x=1$, plus an integral estimating the remaining sum starting at $k=2$ from above.

    Now, we consider the Fourier coefficients $c_k(g)$, $k\in\Z$, of the 1-periodic continuation of $g$ and state an estimate for the absolute values $|c_k(g)|$.
    Applying integration by parts two times we obtain
    \begin{align*}
        c_k(g) &= \int_{-1/2}^{1/2} g(x)\e^{-2\pi\mathrm{i} kx}\mathrm dx
        = \left[g'(x)\frac{\e^{-2\pi\mathrm{i} kx}}{4\pi^2k^2}\right]_{-1/2}^{1/2} - \frac{1}{4\pi^2k^2}\int_{-1/2}^{1/2} g''(x)\e^{-2\pi\mathrm{i} kx} \mathrm dx \\
        &= \frac{(-1)^{k+1}}{4\pi^2k^2\ell^2} \e^{-1/8\ell^2}\left(\frac{1}{8\ell^2}-1\right)
        -\frac{\displaystyle \int_{-\infty}^\infty g''(x)\e^{-2\pi\mathrm{i} kx}\mathrm dx-2\int_{1/2}^\infty g''(x)\cos(2\pi kx)\mathrm dx }{4\pi^2k^2},
    \end{align*}
    where $g''(x)=\ell^{-2}\e^{-x^2/2\ell^2}(1+x^4/(2\ell^4)-5x^2/(2\ell^2))$.
    Thus, by making use of the well-known derivative related properties of the Fourier transform and the triangle inequality, we get
    \begin{equation*}
        |c_k(g)|\leq \frac{\e^{-1/8\ell^2}}{32\pi^2k^2\ell^4} + \frac{\e^{-1/8\ell^2}}{4\pi^2k^2\ell^2}
        %\left|\frac{1}{8\ell^2}-1\right|
        + |\hat g(k)| + \frac{1}{2\pi^2k^2}\int_{1/2}^\infty |g''(x)|\mathrm dx,
    \end{equation*}
    where, $|\hat g(k)|\leq (2\pi^2k^2\ell^2+\frac12) \ell\sqrt{2\pi}\e^{-2k^2\pi^2\ell^2}$, by \eqref{eq:ghat} and \eqref{eq:FT_gauss}.

    In order to estimate the last integral, we examine the sign of the function
    $$g''(x)=\frac{1}{2\ell^2}\left(\frac{x^4}{\ell^4} - \frac{5x^2}{\ell^2}+2\right) \e^{-x^2/2\ell^2}.$$
    Obviously, $g''(0)>0$ and changes its sign in the points
    $$
    x_1 \coloneqq \ell\sqrt{\frac{5-\sqrt{17}}{2}} \text{ and }
    x_2 \coloneqq \ell\sqrt{\frac{5+\sqrt{17}}{2}}
    $$
    with $x_1<x_2$.
    Depending on whether $x_1$ and $x_2$ are smaller or larger than $\tfrac 12$ we obtain the following values.
    If $x_2<\frac 12 \iff \ell<\frac 12\sqrt{\frac{2}{5+\sqrt{17}}}\approx 0.2341$
    \begin{align*}
        \int_{1/2}^\infty |g''(x)|\mathrm dx
        & =-g'(\tfrac 12)= \frac{\e^{-1/8\ell^2}}{2\ell^2} \left(\frac{1}{8\ell^2}-1\right),
    \end{align*}
    where $g'(x)=\frac12\ell^{-2}(2-\ell^{-2}x^2)x\e^{-x^2/2\ell^2}$.
    In the case $x_1<\frac 12<x_2 \iff \ell\in\left(\frac 12\sqrt{\frac{2}{5+\sqrt{17}}},\frac 12\sqrt{\frac{2}{5-\sqrt{17}}}\right)$
    \begin{equation*}
        \int_{1/2}^\infty |g''(x)|\mathrm dx =
        -\int_{1/2}^{x_2} g''(x)\mathrm dx+\int_{x_2}^\infty g''(x)\mathrm dx = \frac{\e^{-1/8\ell^2}}{2\ell^2} \left(1-\frac{1}{8\ell^2}\right)-2g'(x_2).
    \end{equation*}
    Finally, if $x_1>\frac 12 \iff \ell>\frac 12\sqrt{\frac{2}{5-\sqrt{17}}}\approx 0.7551$ and splitting the integral into regions of equal sign leads to
    \begin{equation*}
        \int_{1/2}^\infty |g''(x)|\mathrm dx = 2 g'(x_1) - 2g'(x_2) + \frac{\e^{-1/8\ell^2}}{2\ell^2} \left(\frac{1}{8\ell^2}-1\right).
    \end{equation*}
    For the derivative evaluated in $x_1$ and $x_2$ we compute
    \begin{align*}
        g'(x_1) &= \frac{x_1}{2\ell^2} \left(2-\frac{5-\sqrt{17}}{2}\right)\e^{(-5+\sqrt 17)/4} \approx 0.83 \cdot\frac{1}{2\ell} ,\\
        g'(x_2) &= \frac{x_2}{2\ell^2} \left(2-\frac{5+\sqrt{17}}{2}\right)\e^{(-5-\sqrt 17)/4} \approx -0.56 \cdot\frac{1}{2\ell} ,
    \end{align*}
    and in summary (estimating $0.83<0.9$ and $0.56<0.6$)
    \begin{equation*}
        \int_{1/2}^\infty |g''(x)|\mathrm dx \leq 
        \begin{cases}
            \ds \frac{\e^{-1/8\ell^2}}{2\ell^2} \left(\frac{1}{8\ell^2}-1\right)
            &: \ell\leq\frac12\sqrt{\frac{2}{5+\sqrt{17}}} \quad\text{(I)} \\
            \ds \frac{\e^{-1/8\ell^2}}{2\ell^2} \left(1-\frac{1}{8\ell^2}\right) + \frac{3}{5\ell}
            &: \ell\in\left(\frac 12\sqrt{\frac{2}{5+\sqrt{17}}},\frac 12\sqrt{\frac{2}{5-\sqrt{17}}}\right) \quad\text{(II)}\\
            \ds \frac{\e^{-1/8\ell^2}}{2\ell^2} \left(\frac{1}{8\ell^2}-1\right) + \frac{3}{2\ell}
            &: \ell\geq\frac 12\sqrt{\frac{2}{5-\sqrt{17}}} \quad\text{(III)}
        \end{cases}.
    \end{equation*}
    Putting everything together gives
    \begin{align*}
        |c_k(g)|
        &\leq (2\pi^2k^2\ell^2+\tfrac12)\ell\sqrt{2\pi}\e^{-2\pi^2k^2\ell^2}
        + \frac{1}{2\pi^2k^2}\cdot
        \begin{cases}
            \ds\frac{\e^{-1/8\ell^2}}{8\ell^4} &: \text{case (I)} \\
            \ds\frac{\e^{-1/8\ell^2}}{\ell^2} +\frac{3}{5\ell} &: \text{case (II)} \\
            \ds\frac{\e^{-1/8\ell^2}}{8\ell^4} +\frac{3}{2\ell} &: \text{case (III)}
        \end{cases} \\
        &< (2\pi^2k^2\ell^2+\tfrac12)\ell\sqrt{2\pi}\e^{-2\pi^2k^2\ell^2}
        +\begin{cases}
            \ds\frac{\e^{-1/8\ell^2}}{16\pi^2k^2\ell^4} &: \ell\leq\frac12\sqrt{\frac{2}{5+\sqrt{17}}} \\
            \ds\frac{1}{2\pi^2k^2\ell^2}+\frac{3}{4\pi^2k^2\ell}
            &: \text{else}
        \end{cases},
    \end{align*}
    where we use $x^{-2}\e^{-1/8x^2}<3$ and $\e^{-1/8x^2}<1$ for all $x>0$.
    It follows
    \begin{align*}
        |c_{m/2}(g)| &< (\tfrac12\pi^2m^2\ell^2+\tfrac12)\ell\sqrt{2\pi} \e^{-\pi^2\ell^2m^2/2} \\
        &\; +
        \begin{cases}
            \ds\frac{\e^{-1/8\ell^2}}{4\pi^2m^2\ell^4} &: \ell\leq\frac12\sqrt{\frac{2}{5+\sqrt{17}}} \\
            \ds\frac{2}{\pi^2m^2\ell^2}+\frac{3}{\pi^2m^2\ell} &: \text{else}
        \end{cases}\\
        & = \left(\eta^2+\tfrac12\right)\ell\sqrt{2\pi} \e^{-\eta^2} +
        \begin{cases}
            \ds\frac{\e^{-1/8\ell^2}}{8\eta^2\ell^2} &: \ell\leq\frac12\sqrt{\frac{2}{5+\sqrt{17}}} \\
            \ds\frac{1}{\eta^2}+\frac{3\ell}{2\eta^2} &: \text{else}
        \end{cases} \\
        &\eqqcolon \xi(\eta,\ell).
    \end{align*}
    For the sums to be estimated we obtain by the estimates~\eqref{eq:estimate1_der}, \eqref{eq: se_kernel aux_estimates4} and \eqref{eq: se_kernel aux_estimates3} 
    \begin{align*}
        \sum_{k=m/2+1}^\infty |c_k(g)|
        &< \frac{\erfc(\eta)}{4} + \frac{\sqrt{2\pi}m\ell}{4}\e^{-\eta^2} + \frac{\sqrt{2\pi}}{4m\pi^2\ell}\e^{-\eta^2} \\
        &\; + 
        \begin{cases}
            \ds\frac{\e^{-1/8\ell^2}}{8\pi^2\ell^4m} &: \ell\leq\frac12\sqrt{\frac{2}{5+\sqrt{17}}} \\
            \ds\frac{1}{\pi^2\ell^2m}+\frac{3}{2\pi^2\ell m} &: \text{else}
        \end{cases}\\
        &\leq \frac{\erfc(\eta)}{4} + \frac{\eta}{2\sqrt\pi}\e^{-\eta^2} + \frac{1}{4\sqrt\pi \eta} \e^{-\eta^2} \\
        &\; + 
        \begin{cases}
            \ds\frac{\e^{-1/8\ell^2}}{8\sqrt2\pi\ell^3\eta} &: \ell\leq\frac12\sqrt{\frac{2}{5+\sqrt{17}}} \\
            \ds\frac{1}{\sqrt2\pi\ell\eta}+\frac{3}{2\sqrt2\pi\eta} &: \text{else}
        \end{cases}\\
        &=: S(\eta,\ell)
    \end{align*}
    and by~\eqref{eq:estimate2_der}, \eqref{eq: se_kernel aux_estimates2} and \eqref{eq: se_kernel aux_estimates1}
    \begin{align*}
        \sum_{k=1}^{m/2} |c_k(g)|
        &< \frac14 +
        \begin{cases}
            \frac{\ell\sqrt{2\pi}}{\e} &: \ell\leq\frac{1}{\sqrt2\pi}\\
            2\pi^2\ell^3\sqrt{2\pi}\e^{-2\pi^2\ell^2} &: \text{else}
        \end{cases}
        +\frac14
        + \begin{cases}
            \frac{\e^{-1/8\ell^2}}{96\ell^4} &: \ell\leq\frac12\sqrt{\frac{2}{5+\sqrt{17}}} \\
            \frac{1}{12\ell^2}+\frac{1}{8\ell} &: \text{else}
        \end{cases}.
    \end{align*}
    Now, we can simply compute the maximum possible values regarding the single cases and obtain
    \begin{equation*}
        \sum_{k=1}^{m/2} |c_k(g)|
        <\frac14+\frac12+\frac14+\frac52=\frac72.
    \end{equation*}
    
    The error $\|\kappa_{\text{ERR}}\|_\infty$ can now be estimated by the sum of the non considered Fourier coefficients, see Lemma~\ref{lemma: aliasing error}.
    By making use of the underlying symmetry we obtain
    \begin{align*}
        \|\kappa_{\text{ERR}}\|_\infty
        &\leq 2\cdot \sum_{\bk\in\Z^3\setminus\Im} |c_{k_1}(g)| |c_{k_2}(f)| |c_{k_3}(f)| + \dots + |c_{k_1}(f)| |c_{k_2}(f)| |c_{k_3}(g)|\\
        &\leq\; 6|c_{m/2}(g)| \Bigl(\sum_{k=-m/2}^{m/2} |c_{k}(f)|\Bigr)^2 + 12 |c_{m/2}(f)| \sum_{k=-m/2}^{m/2} |c_k(g)| \sum_{k=-m/2}^{m/2} |c_k(f)| \\
        &\; + 6|c_{m/2}(f)|^2 \sum_{k=-m/2}^{m/2} |c_{k}(g)| + 12|c_{m/2}(g)| |c_{m/2}(f)| \sum_{k=-m/2}^{m/2} |c_{k}(f)| \\
        &\; + 2\cdot3\cdot \sum_{\|\bk\|_\infty \geq m/2+1} |c_{k_1}(g)| |c_{k_2}(f)| |c_{k_3}(f)|\\
        &\eqqcolon 6S_1+12S_2+6S_3+12S_4+6S_5.
    \end{align*}
    We make use of $g(x)\leq \e^{-1}$, implying $|c_0(g)|\leq \e^{-1}<\frac12$, and obtain
    \begin{align*}
        S_1 &< \xi(\eta,\ell)\cdot(\tfrac52)^2, \\
        S_2 &< \gamma(\eta,\ell)\cdot \tfrac{15}{2} \cdot\tfrac52, \\
        S_3 &< \gamma(\eta,\ell)^2\cdot \tfrac{15}{2}, \\
        S_4 &< \xi(\eta,\ell)\cdot \gamma(\eta,\ell)\cdot \tfrac52
    \end{align*}
    and
    \begin{align*}
        S_5 =
        % 8 Ecken
        \;& 8\left(\sum_{k=m/2+1}^\infty |c_k(g)|\right) \left(\sum_{k=m/2+1}^\infty |c_k(f)|\right)^2 \\
        % 12 Kanten
        \,& + 4 \left(\sum_{k=-m/2}^{m/2} |c_k(g)|\right) \left(\sum_{k=m/2+1}^\infty |c_k(f)|\right)^2 \\
        \,& + 4\cdot2\cdot \left(\sum_{k=m/2+1}^\infty |c_k(g)|\right)\left(\sum_{k=m/2+1}^\infty |c_k(f)|\right)\left(\sum_{k=-m/2}^{m/2} |c_k(f)|\right)\\
        % 6 Flächen
        \;&+2 \left(\sum_{k=-m/2}^{m/2} |c_k(f)|\right)^2 \left(\sum_{k=m/2+1}^\infty |c_k(g)|\right) \\
        \;&+ 2\cdot2 \left(\sum_{k=-m/2}^{m/2} |c_k(g)|\right) \left(\sum_{k=-m/2}^{m/2} |c_k(f)|\right) \left(\sum_{k=m/2+1}^\infty |c_k(f)|\right) ,
    \end{align*}
    and with the estimates from above
    \begin{align*}
        S_5 <\;& 8S(\eta,\ell) A(\eta,\ell)^2 + 30 A(\eta,\ell)^2 + 20 S(\eta,\ell) A(\eta,\ell) \\
        \;& + \frac{25}{2} S(\eta,\ell) + 75 A(\eta,\ell) \\
        <\;& \frac{25}{2}S(\eta,\ell) + A(\eta,\ell) \left(\frac{116}{5}S(\eta,\ell)+87\right),
    \end{align*}
    where we simplify $A(\eta,\ell)^2<\frac25 A(\eta,\ell)$, as in the proof of the previous theorem.
    In summary, the derived estimate reads as
    \begin{align*}
        \|\kappa_\text{ERR}\|_\infty
        <\;
        &\frac52 \xi(\eta,\ell) \left(15+12\gamma(\eta,\ell)\right)+ 15\gamma(\eta,\ell) \left(15+3\gamma(\eta,\ell)\right) \\
        &+ 75 S(\eta,\ell) + 6 A(\eta, \ell)\left(\frac{116}{5}S(\eta,\ell)+87\right) \\
        <\;
        &\left(\frac52 \xi(\eta,\ell)+15\gamma(\eta,\ell)\right) \left(15+12\gamma(\eta,\ell)\right) \\
        &+ 75 S(\eta,\ell) + 6 A(\eta, \ell)\left(\frac{116}{5}S(\eta,\ell)+87\right).
    \end{align*}
\end{proof}

We have now established a rigorous error bound for the derivative of one sub-kernel $\kappa_s^{\text{der}}$. The result for the additive kernel $\kappa^{\text{der}}$ follows straightforwardly.

\begin{remark}\label{rem:deriv1d}
    Considering the one-dimensional case with $\kappa^{(1d)}(r)=\frac{r^2}{2\ell^2} \e^{-r^2/(2\ell^2)}$ we obtain
    \begin{align*}
        \|\kappa_{\text{\emph{ERR}}}^{(1d)}\|_\infty
        &\leq 2|c_{m/2}(g)| + 4\sum_{k=m/2+1}^\infty |c_k(g)| \\
        &\leq 2\xi(\eta,\ell) + 4 S(\eta,\ell),
    \end{align*}
    with $\xi(\eta,\ell)$ and $S(\eta,\ell)$ as stated above.

    \begin{figure}[ht]
        \centering
        % l small, value 0.03
        \begin{tikzpicture}\footnotesize
        \begin{axis}
        [width=0.3\linewidth,ymin=-0.1,ymax=0.5,xlabel={$r$}]
        \addplot[color=blue,solid,line width=0.5pt,domain=-0.5:0.5,samples=300]{x^2/(2*0.03^2)*exp(-(x^2/(2*0.03^2))))};
        \addplot[color=blue,solid,line width=0.5pt,domain=-0.7:-0.5,samples=60]{(x+1)^2/(2*0.03^2)*exp(-((x+1)^2/(2*0.03^2))))};
        \addplot[color=blue,solid,line width=0.5pt,domain=0.5:0.7,samples=60]{(x-1)^2/(2*0.03^2)*exp(-((x-1)^2/(2*0.03^2))))};
        \end{axis}
        \end{tikzpicture}
        % l intermediate, value 0.3
        \begin{tikzpicture}\footnotesize
        \begin{axis}
        [width=0.3\linewidth,ymin=-0.1,ymax=0.5,xlabel={$r$}]
        \addplot[color=blue,solid,line width=0.5pt,domain=-0.5:0.5,samples=300]{x^2/(2*0.3^2)*exp(-(x^2/(2*0.3^2))))};
        \addplot[color=blue,solid,line width=0.5pt,domain=-0.7:-0.5,samples=60]{(x+1)^2/(2*0.3^2)*exp(-((x+1)^2/(2*0.3^2))))};
        \addplot[color=blue,solid,line width=0.5pt,domain=0.5:0.7,samples=60]{(x-1)^2/(2*0.3^2)*exp(-((x-1)^2/(2*0.3^2))))};
        \end{axis}
        \end{tikzpicture}
        % l large, value 3
        \begin{tikzpicture}\footnotesize
        \begin{axis}
        [width=0.3\linewidth,ymin=-0.1,ymax=0.5,xlabel={$r$}]
        \addplot[color=blue,solid,line width=0.5pt,domain=-0.5:0.5,samples=300]{x^2/(2*3^2)*exp(-(x^2/(2*3^2))))};
        \addplot[color=blue,solid,line width=0.5pt,domain=-0.7:-0.5,samples=60]{(x+1)^2/(2*3^2)*exp(-((x+1)^2/(2*3^2))))};
        \addplot[color=blue,solid,line width=0.5pt,domain=0.5:0.7,samples=60]{(x-1)^2/(2*3^2)*exp(-((x-1)^2/(2*3^2))))};
        \end{axis}
        \end{tikzpicture}
    \caption{Periodized Gaussian derivative kernels in 1D, with parameters $\ell=0.03$, $\ell=0.3$ and $\ell=3$ (from left to right).
    \label{fig:per_deriv_1d}}
    \end{figure}
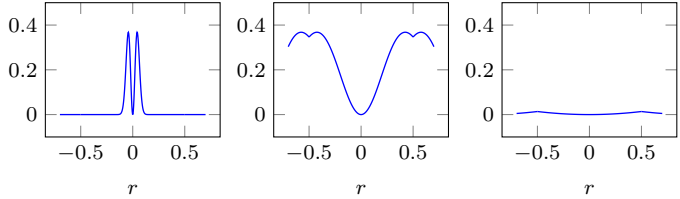
\end{remark}

% Subsubsection: Comparison of Empirical Approximation Error and Analytical Error Estimates
\subsubsection{Comparison of Empirical Approximation Error and Analytical Error Estimates}
In order to check if the actual kernel approximation error can indeed be estimated by the error bounds derived above, we perform the following experiments.
We generate $N=10^4$ uniformly distributed random points in $\br_j\in[-1/2,1/2]^3$.
Then, we evaluate the kernel functions
\begin{align*}
    \kappa^\text{gauss}(\br)=\e^{-\|\br\|^2/2\ell^2}
    \quad\text{ and}\quad
    \kappa^\text{deriv}(\br)=\frac{\|\br\|^2}{2\ell^2} \e^{-\|\br\|^2/2\ell^2}
\end{align*}
in the points $\br_j$, $j=1,\dots,N$, for different values of $\ell$.
In order to approximate the kernel by a trigonometric sum with Fourier coefficients $\hat c_\bk$ we evaluate the kernel function on a regular grid of $m^3$ points in $[-1/2,1/2]^3$, where we select $m\in\{16,32,64\}$.
Finally, we evaluate the obtained trigonometric polynomials in the random points $\br_j$ and compute the measured worst case error via $\max_{j=1,\dots,N} |\kappa_\text{ERR}(\br_j)|$.

In Figure~\ref{fig:error_estimates} the measured errors for different $m$ and $\ell$ are represented by the dotted lines.
The solid lines show the estimates, as presented in Theorems~\ref{theorem: multi_gauss} and \ref{theorem: multi_deriv}.
We can see that the estimated errors are indeed below the corresponding estimates, where for some values of $\ell$ the true error is a few orders of magnitudes smaller than estimated.
However, the error behavior is described qualitatively very well by our estimates.

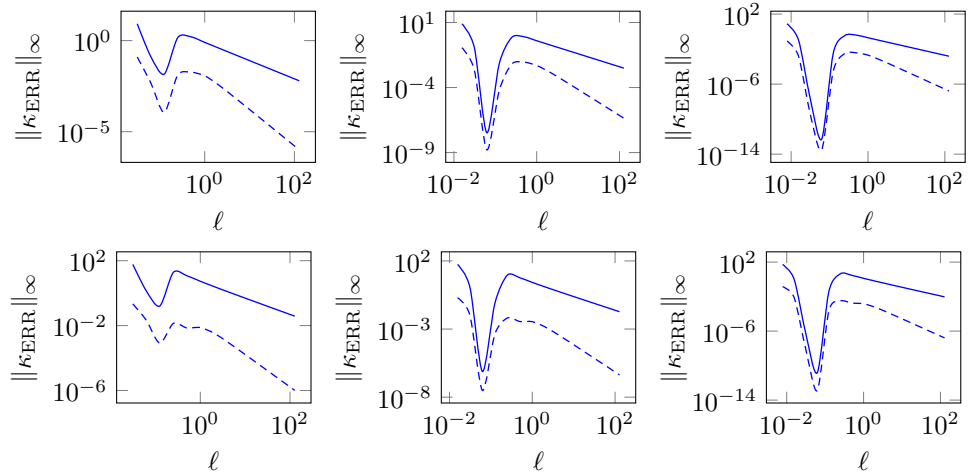
\begin{figure}[ht]
    \centering
    % Gaussian
    \begin{tikzpicture}
        \begin{axis}[width=0.32\linewidth,
        %xmin=0.2,xmax=2.5,ymin=1.0e-15,ymax=10,
        ymode=log,xmode=log,
        legend style={legend pos=south east, nodes=right},
        xlabel={$\ell$},
        ylabel={$\|\kappa_\text{ERR}\|_\infty$}
        ]
        \addplot[smooth,color=blue,line width=0.5pt] table[x=ell,y=gaussM16] {pic/estimates.txt};
        \addplot[smooth,color=blue,line width=0.5pt,densely dashed] table[x=ell,y=gaussM16] {pic/measured_errors.txt};
        %\legend{{$m=16$},{$m=32$},{$m=64$}};
        \end{axis}
    \end{tikzpicture}
    \begin{tikzpicture}
        \begin{axis}[width=0.32\linewidth,
        %xmin=0.2,xmax=2.5,ymin=1.0e-15,ymax=10,
        ymode=log,xmode=log,
        legend style={legend pos=south east, nodes=right},
        xlabel={$\ell$},
        ylabel={$\|\kappa_\text{ERR}\|_\infty$}
        ]
        \addplot[smooth,color=blue,line width=0.5pt] table[x=ell,y=gaussM32] {pic/estimates.txt};
        \addplot[smooth,color=blue,line width=0.5pt,densely dashed] table[x=ell,y=gaussM32] {pic/measured_errors.txt};
        \end{axis}
    \end{tikzpicture}
    \begin{tikzpicture}
        \begin{axis}[width=0.32\linewidth,
        %xmin=0.2,xmax=2.5,ymin=1.0e-15,ymax=10,
        ymode=log,xmode=log,
        legend style={legend pos=south east, nodes=right},
        xlabel={$\ell$},
        ylabel={$\|\kappa_\text{ERR}\|_\infty$}
        ]
        \addplot[smooth,color=blue,line width=0.5pt] table[x=ell,y=gaussM64] {pic/estimates.txt};
        \addplot[smooth,color=blue,line width=0.5pt,densely dashed] table[x=ell,y=gaussM64] {pic/measured_errors.txt};
        \end{axis}
    \end{tikzpicture}
    % deriv. kernel
    \begin{tikzpicture}
        \begin{axis}[width=0.32\linewidth,
        %xmin=0.2,xmax=2.5,ymin=1.0e-15,ymax=10,
        ymode=log,xmode=log,
        legend style={legend pos=south east, nodes=right},
        xlabel={$\ell$},
        ylabel={$\|\kappa_\text{ERR}\|_\infty$}
        ]
        \addplot[smooth,color=blue,line width=0.5pt] table[x=ell,y=derivM16] {pic/estimates.txt};
        \addplot[smooth,color=blue,line width=0.5pt,densely dashed] table[x=ell,y=derivM16] {pic/measured_errors.txt};
        %\legend{{$m=16$},{$m=32$},{$m=64$}};
        \end{axis}
    \end{tikzpicture}
    \begin{tikzpicture}
        \begin{axis}[width=0.32\linewidth,
        %xmin=0.2,xmax=2.5,ymin=1.0e-15,ymax=10,
        ymode=log,xmode=log,
        legend style={legend pos=south east, nodes=right},
        xlabel={$\ell$},
        ylabel={$\|\kappa_\text{ERR}\|_\infty$}
        ]
        \addplot[smooth,color=blue,line width=0.5pt] table[x=ell,y=derivM32] {pic/estimates.txt};
        \addplot[smooth,color=blue,line width=0.5pt,densely dashed] table[x=ell,y=derivM32] {pic/measured_errors.txt};
        %\legend{{$m=16$},{$m=32$},{$m=64$}};
        \end{axis}
    \end{tikzpicture}
    \begin{tikzpicture}
        \begin{axis}[width=0.32\linewidth,
        %xmin=0.2,xmax=2.5,ymin=1.0e-15,ymax=10,
        ymode=log,xmode=log,
        legend style={legend pos=south east, nodes=right},
        xlabel={$\ell$},
        ylabel={$\|\kappa_\text{ERR}\|_\infty$}
        ]
        \addplot[smooth,color=blue,line width=0.5pt] table[x=ell,y=derivM64] {pic/estimates.txt};
        \addplot[smooth,color=blue,line width=0.5pt,densely dashed] table[x=ell,y=derivM64] {pic/measured_errors.txt};
        %\legend{{$m=16$},{$m=32$},{$m=64$}};
        \end{axis}
    \end{tikzpicture}
    \caption{Comparison of the measured estimates $\|\kappa_\text{ERR}\|_\infty$ and corresponding error bounds.
    Results for the Gaussian kernel $\kappa:=\kappa^\text{gauss}$ are depicted in the first line, for the Gaussian derivative kernel $\kappa:=\kappa^\text{deriv}$ in the second line.
    The grid size $m$ in each direction has been set to $m=16$, $m=32$ or $m=64$ (from left to right).
    \label{fig:error_estimates}
    }
\end{figure}

% Subsubsection: Evaluation of Accuracy
\subsubsection{Evaluation of Accuracy}
The ultimate goal of the NFFT-accelerated kernel vector multiplication is to obtain fast ``accurate'' approximate products. For this we want to evaluate the accuracy of the approximation. Let us denote the exact kernel vector product as $p = Kv$, the exact kernel vector product with the Fourier approximation error as $p_{E} = K_{E}v$ and the approximate kernel vector product with $K_{E}$ as $\tilde{p}_{E} \approx K_{E}v$. Then, the overall approximation error is determined by
\begin{align*}
    | p - \tilde{p}_{E} | \le | p - p_{E} | + | p_{E} - \tilde{p}_{E} |.
\end{align*}
Here, $|p - p_{E}|$ describes the Fourier approximation error as introduced above in \eqref{eq: NFFT approx error} and $|p_{E} - \tilde{p}_{E}|$ emerges from employing an approximation algorithm, such as the conjugate gradient method for solving with the kernel matrix for instance. Note that an additional approximation error originates from applying the NFFT as already mentioned above.

%% Section: Global Sensitivity Analysis
\section{Global Sensitivity Analysis} \label{Sec: Global sensitivity analysis}
The analysis of variance (ANOVA) is a concept studied in the context of statistical methods as well as pure and numerical analysis.
In an analytical framework, one may study the so-called classical ANOVA decomposition of functions~\citep{CaMoOw97,LiOw06,KuSlWaWo09}, in order to understand which variables and groups of variables are most important to the function.
Expanding a function by using orthonormal systems makes it easy to decompose its variance by means of the basis coefficients, as presented by \citet{PoSc19a}.
We briefly introduce this concept below.

% Subsection: Sensitivity Analysis in Terms of Fourier Coefficients
\subsection{Sensitivity Analysis in Terms of Fourier Coefficients}
We start with some preliminaries and introduce the required notation.
In the following, we study periodic functions $f:\T^d\to\R$ on the $d$-dimensional torus $\T\simeq[-1/2,1/2)^d$ and denote by 
\begin{equation*}
    [d] \coloneqq \{1,\dots,d\}
\end{equation*}
the set of all dimensions or rather features.
As usual, we denote by $\mathcal P(S)$ the set of all subsets of a set $S$.

Subsets of $[d]$, that is elements of $\mathcal P([d])$, are denoted by small bold letters $\bu$. Such a subset is identified with a vector with ascending entries, for example
\begin{equation*}
    \text{ the subset } \bu=\{1,4,3\} \text{ is identified with the vector } \bu=(1,3,4)\in\N^3.
\end{equation*}
For $\x\in\R^d$ we denote by $\x^\bu\in\R^{|\bu|}$ the restriction of $\x$ to the dimensions present in $\bu$, for example
\begin{equation*}
    \x=(9,8,7,6,5), \bu=\{1,4,3\} \Rightarrow \x^\bu=(9,7,6).
\end{equation*}
Furthermore, we denote by $\mathrm{supp}(\x)$ the set (or vector) of all dimensions $j$ with $x_j\neq0$, for example
\begin{equation*}
    \x=(2,0,0,1,3) \Rightarrow \mathrm{supp}(\x)=(1,4,5).
\end{equation*}
Let $\bv,\bu\in\mathcal P([d])$ with $\bv\subseteq\bu$. Then, we define the elements of the vector $\mathcal F(\bv,\bu)\in \{0,1\}^{|\bu|}$ via 
\begin{equation*}
    \mathcal F(\bv,\bu)_j=\begin{cases}
    1 &: \bu_j=\bv_i \text{ for some } i=1,\dots,|v| \\
    0 &: \text{else}
    \end{cases},
\end{equation*}
where $j=1,\dots,|\bu|$, containing the information which elements of $\bv$ are also included in $\bu$. 
As an example, for $\bv=(1,4)$ and $\bu=(1,2,4)$ we obtain $\mathcal F(\bv,\bu)=(1,0,1)$.

The classical ANOVA decomposition of a function $f\in L_2(\mathbb T^d)$ is a unique decomposition of the form
\begin{equation*}
    f(\x) = \sum_{\bu\subseteq\{1,\dots,d\}} f_\bu(\x^\bu) = f_{\emptyset} + f_{\{1\}}(x_1) + f_{\{2\}}(x_2) + \dots + f_{\{1,\dots,d\}}(\x),
\end{equation*}
consisting of $2^d$ ANOVA terms $f_\bu=f_\bu(\x^\bu)$.
The ANOVA decomposition is defined in such a way that the single ANOVA terms $f_\bu$ are pairwise orthogonal with respect to the usual $L_2$ inner product, that is $\langle f_{\bu},f_{\bv}\rangle=\int_{\T^d} f_\bu(\x)\overline{f_\bv(\x)}\mathrm dx=0$ for $\bu\neq\bv$.
The ANOVA term $f_\emptyset$ is a constant, which equals the mean value of the function.

In the special case of a trigonometric polynomial
\begin{equation}\label{eq:trig_poly}
    f(\x)= \sum_{\bk\in\mathcal I_m} \hat f_{\bk}\, \e^{2\pi\mathrm{i}\bk^\intercal\x},
\end{equation}
one can show that, confer \citet{PoSc19a},
\begin{equation}\label{eq:anova_fourier}
    f_\bu(\x^\bu)
    = \sum_{\substack{\bk\in\mathcal I_m\\ \mathrm{supp}(\bk)=\bu}} \hat f_{\bk}\, \e^{2\pi\mathrm{i}\bk^\intercal\x}
    = \sum_{\substack{\bk\in\mathcal I_m\\ \mathrm{supp}(\bk)=\bu}} \hat f_{\bk}\, \e^{2\pi\mathrm{i}(\bk^\bu)^\intercal\x^\bu}.
\end{equation}
From \eqref{eq:anova_fourier} we see that an ANOVA term $f_\bu$ includes only the frequencies $\bk$, which have non-zero entries on the set of indices $\bu$ and are zero in all dimensions included in $[d]\setminus\bu$, meaning that the ANOVA decomposition introduces a disjoint decomposition of the trigonometric polynomial \eqref{eq:trig_poly} in terms of its Fourier coefficients.

In order to understand the importance of variables and subsets of variables $\bu$ to the function, one studies the variance of $f$ and analyzes the contributions of the single ANOVA terms.
It is well-known that the variance of a trigonometric polynomial is easily determined by summing over the absolute values of the Fourier coefficients, that is
\begin{equation*}
    \sigma^2(f) = \sum_{\bk\in\mathcal I_m\setminus\{\bm0\}} \left|\hat f_\bk\right|^2.
\end{equation*}
For the single ANOVA terms $f_\bu$ we obtain the same, namely
\begin{equation*}
    \sigma^2(f_\bu) = \sum_{\substack{\bk\in\mathcal I_m\\ \mathrm{supp}(\bk)=\bu}} \left|\hat f_\bk\right|^2,
\end{equation*}
so that we conclude
\begin{equation*}
    \sigma^2(f) =\sum_{\substack{\bu\subseteq\{1,\dots,d\}\\\bu\neq\emptyset}} \sigma^2(f_\bu).
\end{equation*}
Based on that, the so-called global sensitivity indices (GSI), confer \citet{So90,So01} and \citet{PoSc19a}, are defined by
\begin{equation*}
    \rho_\bu(f)
    := \frac{\sigma^2(f_\bu)}{\sigma^2(f)} \in[0,1],
\end{equation*}
where we may replace the variances by the sums of the corresponding Fourier coefficients, as explained above.
Non-important subsets $\bu$ will not significantly contribute to the overall variance, meaning $\rho_\bu(f)\approx 0$. In contrast, a large GSI is obtained for important $\bu$. 

% Subsection: Computing the GSI in the Kernel Setting
\subsection{Computing the GSI in the Kernel Setting}
%As described before, our data points are scaled accordingly and the single sub-kernels $\kappa$ are approximated by trigonometric polynomials, see \eqref{Fourier representation}.
%If we denote by $\x,\bm z$ two scaled nodes with $\|\x^\bu-\bm z^\bu\|\leq 1/2\; \forall |\bu|\leq d_\text{max}$, we have
%\begin{equation*}
%    k(\x,\bm z) \approx \sum_{\bk\in\Z^{d_s}} \hat c_\bk\, \e^{2\pi\mathrm{i}\bk^\intercal(\x-\bm z)},
%\end{equation*}
%for which we can easily compute the ANOVA decomposition and GSI, as explained above.
Now, we consider the matrix vector product \eqref{eq: fast sum} with coefficients $v_j$ and $\kappa$ being an additive kernel with windows $\mathcal W_s$, $s=1,\dots,P$, for which $|\mathcal W_s|= d_\text{max}$ holds true. We obtain
\begin{align}
    h(\x)\coloneqq&
    \sum_{j=1}^N v_j \sum_{s=1}^P \kappa_s(\x_j^{\mathcal W_s},\x^{\mathcal W_s}) \notag\\
    \approx& \sum_{j=1}^N v_j \sum_{s=1}^P 
    %\left(
    \sum_{\bk\in\mathcal I_m} \hat c_{\bk}\, \e^{2\pi\mathrm{i}\bk^\intercal(\x_j^{\mathcal W_s}-\x^{\mathcal W_s})}
    %\right)
    \notag \\
    =& \sum_{s=1}^P \sum_{\bk\in\Im}
    \hat c_{\bk} \,S(\bk,\mathcal W_s)
    \,\e^{-2\pi\mathrm{i}\bk^\intercal\x^{\mathcal W_s}} = \tilde h(\x), \label{eq:poly_final}
\end{align}
where $\Im\subset\Z^{d_\text{max}}$ and
$$
S(\bk,\mathcal W_s) \coloneqq \sum_{j=1}^N v_j \e^{2\pi\mathrm{i}\bk^\intercal\x_j^{\mathcal W_s}}.
$$
Note that exactly the same approximation is used for all windows, that is, the set of Fourier coefficients $\{\hat c_{\bk}\}$ is the same for all $\mathcal W_s$.
This is possible since all windows have the same length $d_\text{max}$ and the same length scale parameter $\ell$.
The approximation \eqref{eq:poly_final} is clearly again a trigonometric polynomial and we can now compute the sensitivity indices as explained above.
We summarize the procedure of computing the GSI in this setting in the following algorithm.

%In the most easy case that all kernels have the same weight $c_\bu=1$ \textcolor{red}{reicht uns das?} we have
%$$
%h_\approx(\x)=
%\sum_{\bk\in\bigcup\mathcal I_\bu^{(d)}}
%\tau_\bk \hat f_\bk S(\bk,\bu)
%\e^{-2\pi\i\bk_\bu^\intercal\x_\bu},
%$$
%where $\tau_\bk=|\{\bu: \mathrm{supp}(\bk)\subseteq\bu\}|$ is the number of sub-kernels, where the multiindex $\bk$ is present.

%\textbf{Algorithm (Computation of GSI):} \\
\begin{algorithm}\caption{Computation of GSI}\label{alg:gsi} $\,$ \\
\textbf{Input:} The set of windows $\mathcal W_s\subset\{\bu\subset[d]: |\bu|=d_\text{max}\}$, $s=1,\dots,P$, scaled training data $\x_j$ with $\|\x_j^{\mathcal W_s}\|\leq 1/4$, and corresponding coefficients $v_j$, $j=1,\dots,N$, superposition dimension $d_\text{max}$, length-scale parameter $\ell>0$ and corresponding kernel Fourier coefficients $\hat c_\bk, \bk\in\Im\subset\Z^{d_\text{max}}$ (precomputed via periodization and FFT).
\begin{enumerate}
    \item For all $\bu \in \bigcup_{s=1}^P\mathcal P(\mathcal W_s)$ initialize $\theta_{\bu}:=0$.
    \item\label{step2} For all $s=1,\dots,P$ do:
        \begin{enumerate}
            %\item $\theta_s \coloneqq 0$
            \item Compute $S(\bk,\mathcal W_s)$, $\bk\in\mathcal I_m$ (this is an adjoint or rather type-2 NFFT).
            \item For all $\emptyset\neq\bv\in\mathcal P(\mathcal W_s)$ compute
            $$
            \theta_\bv =
            \theta_\bv + \sum_{\mathrm{supp}(\bk)=\mathrm{supp}(\mathcal F(\bv,\mathcal W_s))} \left| \hat c_\bk S(\bk,\mathcal W_s)\right|^2.
            $$
        \end{enumerate}
    \item Compute the overall variance
    $$
    \sigma^2(\tilde h) \coloneqq
    \sum_{\emptyset\neq\bv\in\bigcup_{s=1}^P\mathcal P(\mathcal W_s)} \theta_\bv.
    $$
    \item For all $\emptyset\neq\bv\in\bigcup_{s=1}^P\mathcal P(\mathcal W_s)$ compute the GSI via
    $$\rho_\bv(\tilde h) \coloneqq \frac{\theta_\bv}{\sigma^2(\tilde h)}.$$
\end{enumerate}
\textbf{Output:} Global sensitivity indices $\rho_\bv(\tilde h)$ for all $\bv\in\bigcup_{s=1}^P \mathcal P(\mathcal W_s) \setminus \emptyset$, that is, for all given windows $\mathcal W_s$ and all their subsets, except for $\bv=\emptyset$.
\end{algorithm}
Note that in Algorithm~\ref{alg:gsi} we consider the special case where all given windows in the kernel have exactly the same cardinality, that is $|\mathcal W_s|=d_\text{max}$.
The case $|\mathcal W_s|\leq d_\text{max}$ can be realized analogously and is not more complicated. The notation will be slightly more complex for this more general case, since the set of Fourier coefficients $\{\hat c_\bk, \bk\in\Im\}$ differs for windows of different lengths.
%If windows with different lengths are given, we need $d_\text{max}$ many sets of Fourier coefficients, whether $\mathcal I_m^{(1)}\subset \Z^1$, \dots, $\mathcal I_m^{(d_\text{max})}\subset\Z^{d_\text{max}}$ and the notation just becomes somewhat more complicated.
We would like to mention that the adjoint NFFT that has to be computed in step~\ref{step2} in the above algorithm is computed using the \texttt{pynufft}\footnote{\url{https://github.com/jyhmiinlin/pynufft}} software package in our Python codes.

\subsection{Variation of the GSI scores}

As explained above, sensitivity indices are computed for all subsets of features of cardinality smaller or equal $d_{\text{max}}$. Those subsets are then sorted by their GSI in descending order. Since the sum of those indices is $1$ over all feature subsets, we define a $\texttt{GSI}_{\text{score}} \in (0,1)$ determining how many of those subsets shall be assigned to the feature window. Starting with the subset with the largest GSI, subsets are added to the feature window until the sum of GSI reaches $\texttt{GSI}_{\text{score}}$. The larger the GSI score the more feature subsets are selected. Of course, $\texttt{GSI}_{\text{score}}$ has to be selected carefully in order to obtain good feature windows. The optimal choice can vary for different data sets and is not straightforward. In Figure~\ref{fig:gsi_scores} we compare the RMSE, window size and runtime yielded by models with windows generated for different GSI scores. As expected the RMSE increases when increasing the GSI score up to a certain level. At some point, adding more feature subsets to the window does not lead to better prediction quality. Interestingly, the number of features and windows included in the model is equal for most GSI scores. Only for very large GSI scores, the number of windows goes up steeply. The time for running the model with the corresponding windows behaves accordingly. The prediction quality of the additive model with windows generated with global sensitivity analysis clearly outperforms the full KRR model for the two data sets considered. Note that the performance of sklearn KRR with the full kernel could likely be improved by a more exhaustive grid search. In the following, we set $\texttt{GSI}_{\text{score}}=0.99$, as we achieve a very high prediction quality with this score and can keep the number of features and windows involved and thus the runtime moderate.

% Begin figure GSI
\begin{figure}
    \centering
    \includegraphics[width=0.65\textwidth]{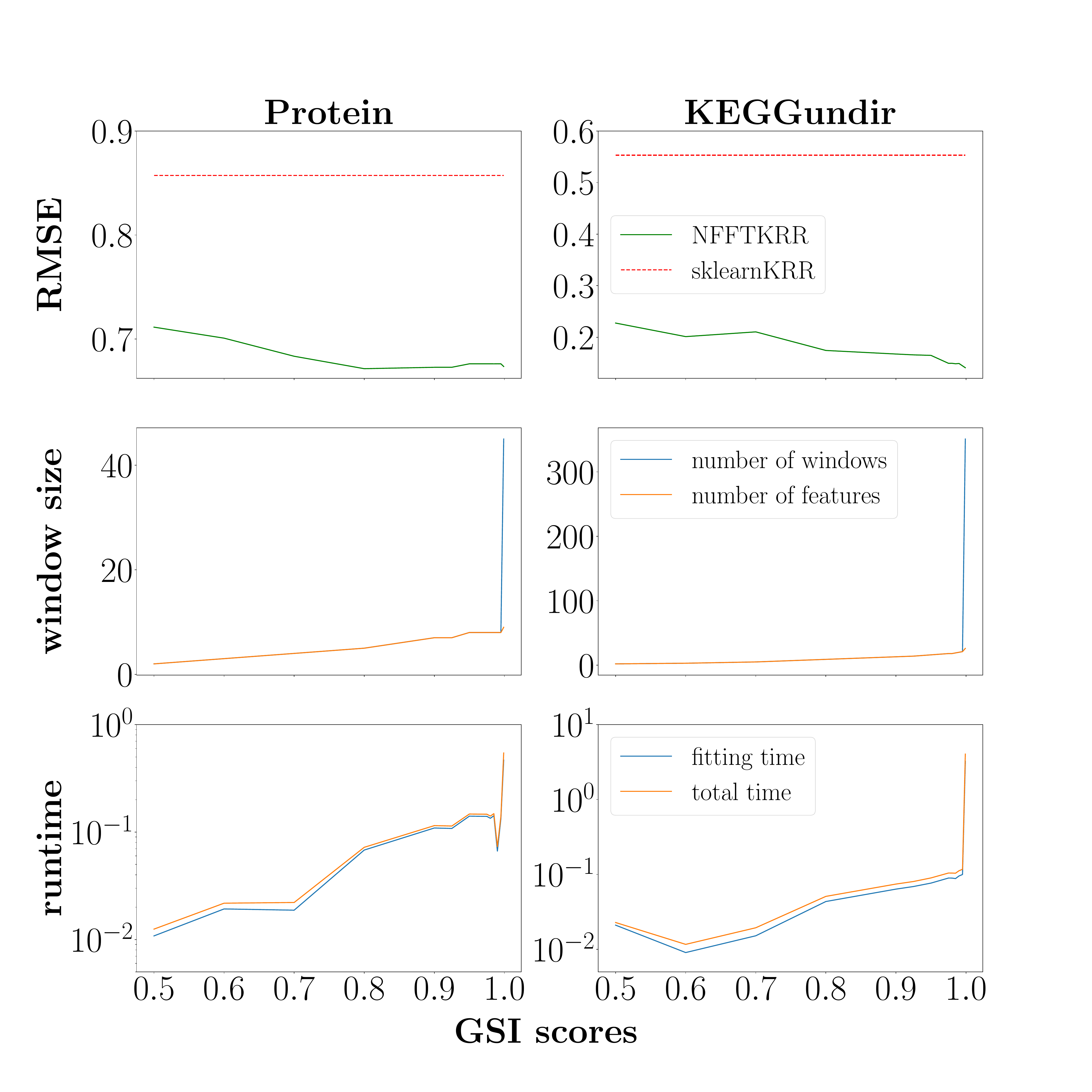}
    \caption{Comparison of RMSE, window size and runtime for the additive KRR model for different GSI scores with $N=1000$, $d_{\text{max}}=3$, $N_{\text{feat}}=d$ and initial $\ell=1$ and $\beta=1$.}
    \label{fig:gsi_scores}
\end{figure}
% End figure GSI

%% Section: Numerical Results
\section{Numerical Results} \label{Sec:Numerical results}

To demonstrate the predictive power of the feature arrangement techniques presented above we perform additive kernel ridge regression on benchmark data sets with NFFT-approximations. The corresponding implementations are available in the GitHub repository \texttt{NFFTAddKer}, see \url{https://github.com/wagnertheresa/NFFTAddKer}. The underlying repository for the fast NFFT-based kernel evaluations is \texttt{prescaledFastAdj} as introduced above that accesses parts of the \texttt{NFFT} library.

In the following we compare the results of the more sophisticated with the basic techniques to examine whether it is worth putting more effort into the preprocessing phase of learning the windows $\mathcal{W}_s$ and whether additive kernels actually allow for higher accuracy.

Furthermore, we investigate whether the intuition holds true that feature groups covering feature interactions incorporate more information into the model what leads to higher prediction accuracy than groups consisting of single features only.

% Subsection: Experimental Setup
\subsection{Experimental Setup}

All experiments were run on a computer with $8 \, \times$ Intel Core $\text{i}7-7700$ CPU @ $3.60$ GHz processors with NV106 graphics and $16.0$ GiB of RAM. We consider the UCI data sets Protein~\citep{physicochemical_properties_of_protein_tertiary_structure_265} ($N=45730$, $d=9$), KEGGundir~\citep{kegg_metabolic_reaction_network_(undirected)_221} ($N=63608$, $d=26$) and Bike Sharing~\citep{bike_sharing_275} ($N=17379$, $d=14$) and the StatLib data set Housing~\citep{pace1997sparse} ($N=20640$, $d=8$). Note that data points with missing entries or entries exceeding the range defined for the feature are dropped in the KEGGundir data set. The data is z-score normalized, the labels are transformed to normalize the target distribution and the data and length-scale parameters are prescaled as described in Subsection~\ref{Sec:Scaling}. We perform grid search for the additive kernel ridge regression. All results presented in this paper were generated with the parameter choices listed in Table~\ref{tab:params} unless stated otherwise. 

% Table with experimental setup
\begin{table}
\centering
\begin{tabular}{|r|r l|}
 \hline
\multicolumn{3}{| c |}{General parameter setting}\\
 \hline
 train--test split & $\texttt{data}_{\text{split}}=$ & \hspace{-0.9em}$0.5$ \\
 signal variance parameter & $\sigma_f=$ & \hspace{-0.9em}$\sqrt{1/P}$ \\
 CG convergence tolerance & $\texttt{tol}_{\text{CG}}=$ & \hspace{-0.9em}$10^{-3}$ \\
 NFFT parameter setup & $\texttt{setup}_{\text{NFFT}}=$ & \hspace{-0.9em}``default'' \\ \hline
 \multicolumn{3}{| c |}{Parameter setting for feature grouping techniques} \\
 \hline
 subset size for feature grouping & $N^{\text{fg}}=$ & \hspace{-0.9em}$1000$ \\
 subset size for FGO & $N_{\text{FGO}}^{\text{fg}}=$ & \hspace{-0.9em}$500$ \\
 threshold for dropping features & $\texttt{thres}=$ & \hspace{-0.9em}$0.0$ \\
 L$1$ regularization parameter for lasso & $\beta_{\text{lasso}}^{\text{L}1}=$ & \hspace{-0.9em}$0.01$ \\
 L$1$ regularization parameter for EN & $\beta_{\text{EN}}^{\text{L}1}=$ & \hspace{-0.9em}$0.01$ \\
 L$1$ ratio for EN & $\texttt{ratio}_{\text{EN}}^{\text{L}1}=$ & \hspace{-0.9em}$0.5$ \\
 fixed length-scale parameter for FGO & $\ell_{\text{FGO}}=$ & \hspace{-0.9em}$1$ \\
 fixed regularization parameter for FGO & $\beta_{\text{FGO}}=$ & \hspace{-0.9em}$0.1$ \\
 GSI score & $\texttt{GSI}_{\text{score}}=$ & \hspace{-0.9em}$0.99$ \\
 initial length-scale parameter for GSI & $\ell_{\text{GSI}}^{\text{init}}=$ & \hspace{-0.9em}$1$ \\
 initial regularization parameter for GSI & $\beta_{\text{GSI}}^{\text{init}}=$ & \hspace{-0.9em}$1$ \\ \hline
 \multicolumn{3}{| c |}{Candidate parameter values for grid search} \\
 \hline
 length-scale parameter & $\ell \in $ & \hspace{-0.9em}$[ 10^{-2}, 10^{-1}, 1, 10^1, 10^2]$ \\
 regularization parameter & $\beta \in $ & \hspace{-0.9em}$[ 10^{-2}, 10^{-1}, 1, 10^1, 10^2]$ \\ \hline
\end{tabular}
\caption{Parameter setting for the experiments presented in this paper.}
\label{tab:params}
\end{table}

Other parameters that have to be chosen are the maximal length of the windows $d_{\text{max}}$ and the total number of features included $N_{\text{feat}}$ that are required for the feature arrangement techniques based on a feature importance ranking. In the remainder of this section we analyze how the choice of these parameters affects the performance of the corresponding regression model. Finally, we compare the feature importance ranking based methods to the approaches based on optimization and global sensitivity analysis.

In this section, we examine the following feature arrangement techniques: consecutive feature grouping (consec), decision tree (DT), mutual information score (MIS), Fisher score (Fisher), RreliefF as filter (relfilt) and wrapper (relwrap) method, lasso, elastic net (EN), feature clustering based on connected components (FC CC), feature grouping optimization (FGO) and global sensitivity indices (GSI).

% Subsection: Variation of the Maximal Window Length
\subsection{Variation of the Maximal Window Length}

As motivated above, fast approximation techniques can only exploit their full computational power in small feature spaces. Therefore, a maximal window length $d_{\text{max}}$ must be defined to determine the windows accordingly. For the NFFT-accelerated approximation $d_{\text{max}}$ shall be smaller than $4$.

% Begin figure dmax
\begin{figure}
    \centering
    \includegraphics[width=\textwidth]{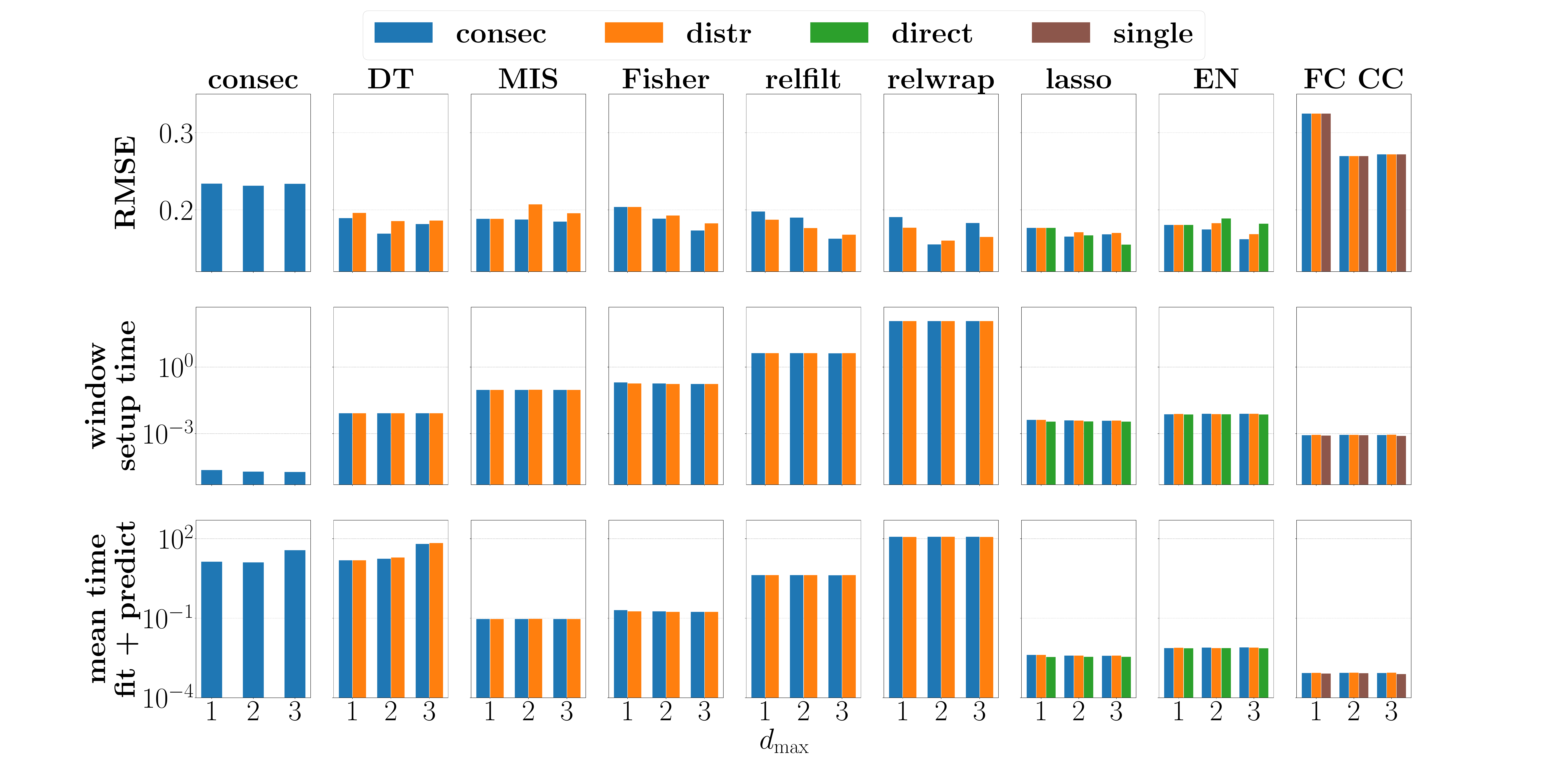}
    \caption{Comparison of RMSE, window setup time and time for fitting and predicting the additive KRR model with the corresponding windows for different feature arrangement techniques and strategies, fixed number of total features included $N_{\text{feat}}=2d/3$ and different maximal window length $d_{\text{max}}$ for the KEGGundir data set.}
    \label{fig:dmax}
\end{figure}
% End figure dmax

In Figure~\ref{fig:dmax} we analyze the impact of its value on the feature importance ranking based techniques for different arrangement strategies, where the total number of features involved is fixed to two-thirds of $d$. We compare the RMSE of the additive regression model obtained with the corresponding windows, the time for determining the windows and the mean time for fitting and predicting the model in the grid search routine for the KEGGundir data set. In most RMSE plots the bars shrink the larger $d_{\text{max}}$. While FC CC clearly has the best time for fitting and predicting and yields the second best windows setup time, it cannot keep up with the competitors regarding RMSE. relfilt and relwrap take the longest for setting up the windows and are among the slower methods for fitting and predicting. The corresponding RMSE is in the midfield but cannot compensate for the high runtimes however. Lasso and EN provide the second best runtime for fitting and predicting, the third best window setup time and are among the best in RMSE. DT provides one of the best RMSE results and is in the midfield in the window setup time. However, the fitting and predicting takes one of the longest. In comparison, MIS and Fisher yield quite similar RMSE values as DT but take longer for generating the windows. Fitting and predicting is faster by several orders of magnitude though. In total, MIS performs slightly better than Fisher in all categories. Naturally, consec is fastest in determining the windows. The RMSE is far from the best and fitting and predicting is among the slowest. While some feature arrangement strategies beat others in particular techniques, no clear trend of one of them outperforming the others can be identified.

As expected the choice of $d_{\text{max}}$ mostly does not impact the window setup time. However, it generally does not strongly affect the time for fitting and predicting the model either. For most techniques the runtime increases by factor $2$ to $4$ when changing $d_{\text{max}}$ from $1$ to $3$, what is barely visible in the figure. The larger $d_{\text{max}}$ the more Fourier coefficients have to be computed per sub-kernel $K_s$. A smaller value of $d_{\text{max}}$ however leads to a larger number of windows and sub-kernels $P$ for a fixed number $N_{\text{feat}}$. Therefore, both aspects mostly balance each other out.

% Subsection: Variation of the Total Number of Features Included
\subsection{Variation of the Total Number of Features Included}

The experiment on the maximal window length has illustrated that $d_{\text{max}}=3$ can be a good choice since it usually yields the smallest RMSE while it only leads to an insignificantly greater computational effort. Next, we investigate how the total number of features included impacts the overall performance for fixed $d_{\text{max}}=3$.

% Begin figure Nfeat
\begin{figure}
    \centering
    \includegraphics[width=\textwidth]{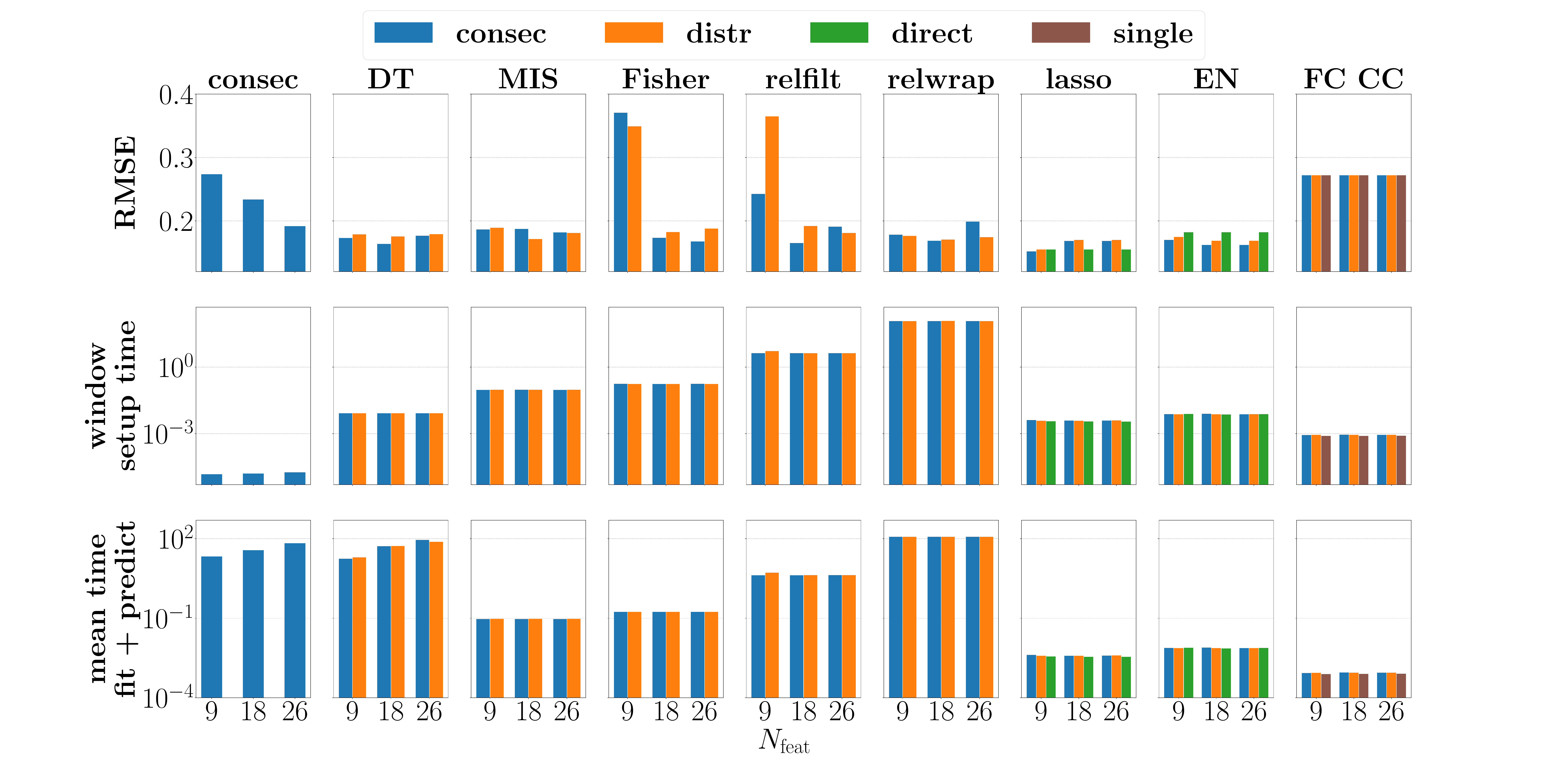}
    \caption{Comparison of RMSE, window setup time and time for fitting and predicting the additive KRR model with the corresponding windows for different feature arrangement techniques and strategies, fixed $d_{\text{max}}=3$ and different number of total features included $N_{\text{feat}}$ for the KEGGundir data set.}
    \label{fig:Nfeat}
\end{figure}
% End figure Nfeat

Figure~\ref{fig:Nfeat} shows the performance of the feature importance ranking based techniques for different arrangement strategies, fixed $d_{\text{max}}=3$ and different values $d/3$, $2d/3$ and $d$ for $N_{\text{feat}}$. The runtime plots behave similarly as the ones in Figure~\ref{fig:dmax} for the different feature arrangement techniques. Again, we cannot recognize that one of the arrangement strategies is superior to the other ones and the choice of $N_{\text{feat}}$ does not seem to have an impact on the time for running the model. For all but one technique, the RMSE is largest for $N_{\text{feat}}=d/3$ and smallest for $N_{\text{feat}}=2d/3$. In most cases, the RMSE for $N_{\text{feat}}=d$ is either at the same level or larger than for $N_{\text{feat}}=2d/3$.

% Subsection: Comparison to GSI, FGO and Full Kernel Ridge Regression
\subsection{Comparison to GSI, FGO and Full Kernel Ridge Regression}

In the previous subsections we observed that using lasso and EN to determine the feature windows usually led to the smallest RMSE. Moreover, the window setup time and the time for fitting and predicting the model with the corresponding windows is superior to most of the other methods. MIS can be considered as the best technique that is not based on a regularization. The only other method that could keep up with MIS is DT that reached similar RMSE. While the window setup time of DT is actually smaller than for MIS, the time for running the model is larger by up to $3$ orders of magnitude. Even though none of the feature arrangement strategies is clearly preferable for those techniques, `distr' might be slightly the best for MIS, `direct' for lasso and `consec' for EN.

In Figure~\ref{fig:final_comparison}, we compare those leading feature importance ranking based techniques MIS, lasso and EN to GSI, FGO and the state-of-the-art sklearn kernel ridge regression with the full kernel on $4$ benchmark data sets.

% Begin figure final comparison
\begin{figure}
    \centering
    \includegraphics[width=\textwidth]{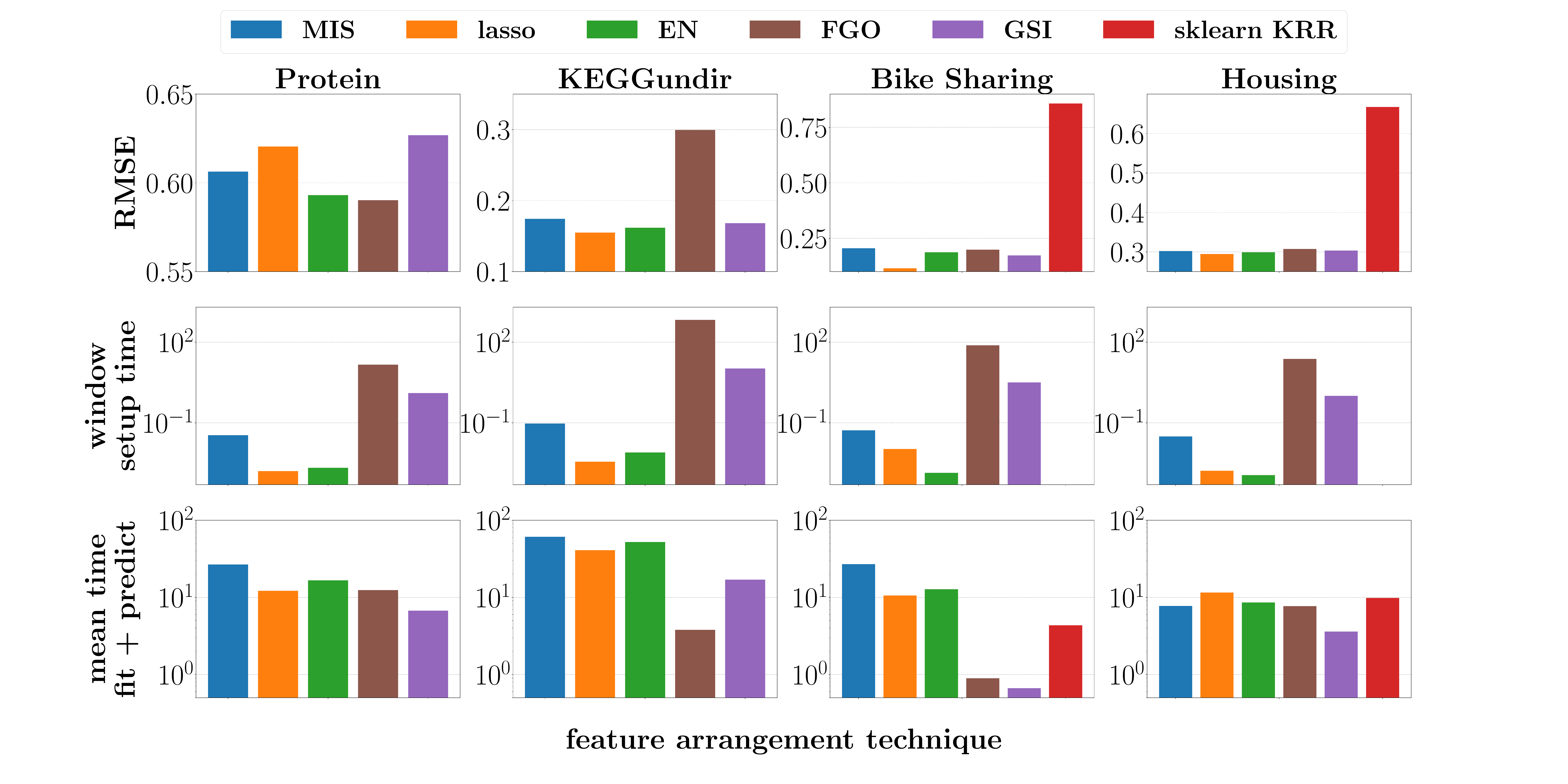}
    \caption{Comparison of the feature importance ranking based techniques MIS, lasso and EN with FGO and GSI for the additive KRR model, and sklearn KRR on the full kernel, with $d_{\text{max}}=3$, $N_{\text{feat}}^{\text{Protein}}=N_{\text{feat}}^{\text{Bike}}=9$, $N_{\text{feat}}^{\text{KEGGundir}}=18$, $N_{\text{feat}}^{\text{Housing}}=8$, $\beta_{\text{FGO}}^{\text{Protein}}=2.5$, $\beta_{\text{FGO}}^{\text{KEGGundir}}=0.5$, $\beta_{\text{FGO}}^{\text{Bike}}=1.0$, $\beta_{\text{FGO}}^{\text{Housing}}=1.5$.}
    \label{fig:final_comparison}
\end{figure}
% End figure final comparison

Note that other than for the KEGGundir data set $N_{\text{feat}}=d$ can lead to a further RMSE improvement for other data sets, in particular when $d$ is small for instance. Therefore, we choose $N_{\text{feat}}=2d/3$ for the KEGGundir and Bike Sharing data set and $N_{\text{feat}}=d$ for the Protein and Housing data set. Moreover, we set $d_{\text{max}}=3$ for the feature importance ranking based techniques and refer to Table~\ref{tab:params} for the further parameter setting. For all 4 data sets considered, EN performs better than MIS in all three categories. Comparing lasso and EN, we cannot recognize an obvious trend of one outperforming the other. For the Bike Sharing and Housing data, lasso yields a better RMSE than EN but worse runtimes and for the Protein data set EN clearly returns a better RMSE but slightly worse runtimes. As expected, FGO and GSI require by far the longest window setup time but often yield a very competitive runtime for training and fitting the model. Once the windows are set up with these methods running the model is usually quite efficient since FGO and GSI usually return fewer windows of shorter length since they are not affected similarly by the choices of $d_{\text{max}}$ and $N_{\text{feat}}$. However, the RMSE obtained with those windows cannot always keep up with the competitors. Note that a careful adjustment of the model parameters in FGO and GSI can lead to a competitive RMSE for the Protein and KEGGundir data sets. The RMSE obtained with windows generated via GSI is already competitive for the KEGGundir, Bike Sharing and Housing data sets. Since GSI windows usually incorporate many features separately, the intuition that windows with larger $d_s$ yield better RMSE cannot be confirmed in general. The red bar represents the performance of sklearn KRR with the full kernel. The additive models clearly provide a better RMSE that is more than two or three times smaller than for the full KRR model. The time for fitting and predicting the model can be slightly smaller if the number of data points is small such as for the Bike Sharing data set. As motivated above in Figure~\ref{fig:nfft_approx_error}, the computational complexity of NFFT-based additive kernel evaluations is evidently smaller for large scale problems. Note that the red bar is missing in the Protein and KEGGundir plots since the computations break for bigger matrices in the sklearn KRR model due to a lack of memory.

% Subsection: Extension to Other Kernels
\subsection{Extension to Other Kernels}

Naturally, the NFFT-accelerated kernel matrix and derivative kernel evaluations and the presented feature arrangement techniques are not only tailored to the Gaussian kernel but can be applied to other kernels directly. One of the most popular alternatives is the Mat\'ern$(\frac{1}{2})$ kernel
\begin{align*}
    \kappa_{\text{Mat\'ern}}^{1/2} ( \bm{x}, \bm{x}') = \exp \left( - \frac{\| \bm{x} - \bm{x}' \|_2}{\ell} \right).
\end{align*}
Analogously to~\eqref{eq:wind_gauss_anova_kernel} for the additive Gaussian kernel, we can define the Mat\'ern$(\frac{1}{2})$ kernel additively as
\begin{align*} \label{eq:wind_matern_kernel}
    \kappa_{\text{M}}^{1/2} ( \bm{x}_i , \bm{x}_j ) = \sigma_f^2 \sum_{s=1}^P \underbrace{\exp \left( - \frac{\| \bm{x}_i^{\mathcal{W}_s} - \bm{x}_j^{\mathcal{W}_s} \|_2}{\ell} \right)}_{\kappa_{\text{M}_s}^{1/2}}.
\end{align*}
In the \texttt{prescaledFastAdj} repository $\kappa_{\text{Mat\'ern}}^{1/2}$ is referred to as $\text{kernel}=3$ and embedded as `laplacian\_rbf' in the underlying \texttt{NFFT} repository~\citep{NFFTrepo}. Differentiation with respect to the signal variance parameter $\sigma_f$ gives
\begin{align*}
    \frac{\partial K_{{\text{M}}_{ij}}^{1/2}}{\partial \sigma_f} = 2 \sigma_f \sum_{s=1}^P \exp \left( - \frac{\| \bm{x}_i^{\mathcal{W}_s} - \bm{x}_j^{\mathcal{W}_s} \|_2}{\ell} \right) = \frac{2}{\sigma_f} K_{{\text{M}}_{ij}}^{1/2}
\end{align*}
and with respect to the length-scale parameter $\ell$ we obtain
\begin{align*}
    \frac{\partial K_{{\text{M}}_{ij}}^{1/2}}{\partial \ell} = \sigma_f^2 \sum_{s=1}^P \frac{ \| \bm{x}_i^{\mathcal{W}_s} - \bm{x}_j^{\mathcal{W}_s} \|_2}{\ell^2} \exp \left( - \frac{\| \bm{x}_i^{\mathcal{W}_s} - \bm{x}_j^{\mathcal{W}_s} \|_2}{\ell} \right) = \sigma_f^2 \sum_{s=1}^P \frac{C_{\text{M}_s}}{\ell^2} \circ K_{{{\text{M}}_s}_{ij}}^{1/2},
\end{align*}
with $C_{{\text{M}_s}_{ij}} = \| \bm{x}_i^{\mathcal{W}_s} - \bm{x}_j^{\mathcal{W}_s} \|_2$. For the latter we added the `der\_laplacian\_rbf' kernel
\begin{align*}
    \kappa_{\text{der}\text{M}_s}^{1/2} (\bm{x}_i, \bm{x}_j) = \frac{\| \bm{x}_i^{\mathcal{W}_s} - \bm{x}_j^{\mathcal{W}_s} \|_2}{\ell} \exp \left( - \frac{\| \bm{x}_i^{\mathcal{W}_s} - \bm{x}_j^{\mathcal{W}_s} \|_2}{\ell} \right)
\end{align*}
to the \texttt{NFFT} repository that is referred to as $\text{kernel}=4$ within \texttt{prescaledFastAdj}. With that, we obtain $\frac{\partial K_{\text{M}}^{1/2}}{\partial \sigma_f} v = 2\sigma_f \sum_{s=1}^P K_{\text{M}_s}^{1/2} v$ and $\frac{\partial K_{\text{M}}^{1/2}}{\partial \ell} v = \sigma_f^2 \sum_{s=1}^P \frac{1}{\ell} K_{\text{der}\text{M}_s}^{1/2} v$.

% Begin figure final comparison Matérn kernel
\begin{figure}
    \centering
    \includegraphics[width=\textwidth]{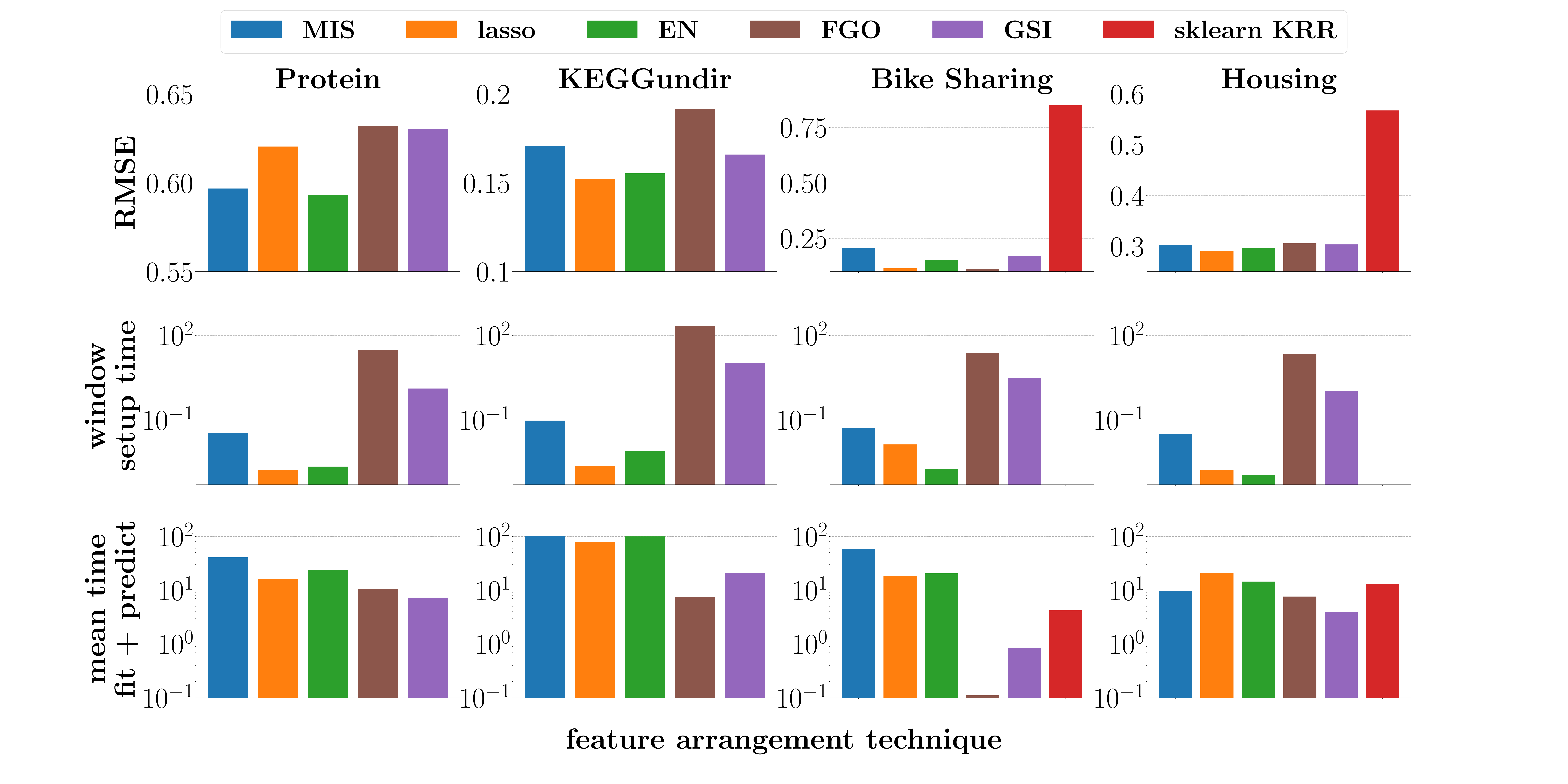}
    \caption{Comparison of the feature importance ranking based techniques MIS, lasso and EN with FGO and GSI for the additive KRR model, and sklearn KRR on the full kernel, with $d_{\text{max}}=3$, $N_{\text{feat}}^{\text{Protein}}=N_{\text{feat}}^{\text{Bike}}=9$, $N_{\text{feat}}^{\text{KEGGundir}}=18$, $N_{\text{feat}}^{\text{Housing}}=8$, $\beta_{\text{FGO}}^{\text{Protein}}=2.5$, $\beta_{\text{FGO}}^{\text{KEGGundir}}=0.5$, $\beta_{\text{FGO}}^{\text{Bike}}=1.0$, $\beta_{\text{FGO}}^{\text{Housing}}=1.5$ for the Mat\'ern$(\frac{1}{2}$) kernel.}
    \label{fig:final_comparison_matern}
\end{figure}
% End figure final comparison Matérn kernel

In Figure~\ref{fig:final_comparison_matern} we compare the performance of different feature arrangement techniques for an additive KRR model working with the Mat\'ern$(\frac{1}{2}$) instead of the Gaussian kernel as in Figure~\ref{fig:final_comparison}. In comparison to the model with the Gaussian kernel, the Mat\'ern$(\frac{1}{2}$) kernel does not lead to huge differences in the runtimes for fitting and predicting the model. The different kernel definition does not modify the MIS, lasso and EN techniques but also for FGO and GSI we cannot recognize huge variations in the window setup time in comparison to Figure~\ref{fig:final_comparison}. The RMSE plots however show greater alternation. Especially for the FGO technique, the RMSE obtained with the Mat\'ern$(\frac{1}{2}$) kernel can improve as for the KEGGundir and Bike Sharing data set but also deteriorate as for the Protein data set. The other feature arrangement techniques do not show great variations in performance between the two kernels.

It is also possible to include further kernels such as the Mat\'ern$(\frac{3}{2}$) by specifying the function and its derivatives within the \texttt{NFFT} package. It remains to derive the corresponding Fourier error estimates in future work.

%% Section: Conclusion
\section{Conclusion} \label{Sec:Conclusion}

In this paper we have analyzed feature arrangement techniques for additive regression models and their applicability to NFFT-accelerated kernel evaluations. We presented several options for splitting the original feature space into smaller feature groups and examined their performance. For simplicity, we demonstrated the numerical results on an additive KRR model, but the computations can easily be applied to other kernel methods.  Moreover, we developed an NFFT-acceleration procedure for kernel evaluations with the derivative kernel and motivated its computational power empirically. This is of great relevance in hyperparameter optimization tasks for GPs, for instance. We derived the corresponding Fourier error estimates for the trivariate Gaussian kernel and its derivative kernel analytically and demonstrated its quality. Finally, we compared the additive KRR model to the state-of-the-art sklearn KRR model with the full kernel matrix. In our experiments, the additive model could consistently yield clearly better RMSE while requiring smaller runtimes for fitting and predicting the model if the data is large enough. We mostly focused on the Gaussian kernel in this paper, but briefly motivate the extension to other kernels such as the Mat\'ern$(\frac{1}{2}$) kernel and present first numerical results. It remains to derive additional theoretical guarantees for this kernel in future work.

% Manual newpage inserted to improve layout of sample file - not
% needed in general before appendices/bibliography.

\section*{Acknowledgments}
The authors gratefully acknowledge their support from the Bundesministerium f\"{u}r Bildung und Forschung (BMBF) grant 01\,$\vline$\,S20053A (project SA$\ell$E).

\vskip 0.2in
\bibliographystyle{plainnat}
%\bibpunct{(}{)}{;}{a}{,}{,}
%\bibliography{references}

\end{document}